\documentclass{article}
\usepackage[utf8]{inputenc}

\usepackage{amsmath,amsthm,amssymb,mathtools,algorithm,tikz,xcolor}
\usepackage{enumitem}
\usepackage{caption}
\usepackage{subcaption}
\usepackage{algorithm}
\usepackage[noend]{algorithmic}
\usepackage{thm-restate}
\usepackage[algo2e,ruled,vlined]{algorithm2e}
\usepackage[colorlinks,citecolor=blue,urlcolor=blue,bookmarks=false,hypertexnames=true]{hyperref} 
\usepackage[numbers]{natbib} 

\usepackage{arxiv}

\newcommand{\R}{\mathbb{R}}
\newcommand{\E}{\mathbb{E}}
\newcommand{\bI}{\mathbb{I}}
\newcommand{\bP}{\text{Pr}}

\newcommand{\C}{\mathcal{C}}

\newcommand{\G}{\mathcal{G}}

\newcommand{\h}{\mathcal{H}}
\newcommand{\s}{\mathcal{S}}
\newcommand{\F}{\mathcal{F}}
\newcommand{\U}{\mathcal{U}}
\newcommand{\X}{\mathcal{X}}
\newcommand{\Y}{\mathcal{Y}}
\newcommand{\D}{\mathcal{D}}
\newcommand{\cP}{\mathcal{P}}

\newcommand{\cS}{\mathcal{S}}
\newcommand{\norm}[1]{\left\lVert#1\right\rVert}

\DeclareMathOperator*{\argmin}{arg\,min}
 
\newtheorem{theorem}{Theorem}

\newtheorem{defn}[theorem]{Definition}

\newtheorem{lemma}[theorem]{Lemma}

\newtheorem{remark}{Remark}
\newtheorem*{fact}{Fact}

\title{Data driven semi-supervised learning}
\author{
  Maria-Florina Balcan\\
  School of Computer Science\\
  Carnegie Mellon University\\
  Pittsburgh, PA 15213 \\
  \texttt{ninamf@cs.cmu.edu} \\
  \And
  Dravyansh Sharma\\
  Department of Computer Science\\
  Carnegie Mellon University\\
  Pittsburgh, PA 15213 \\
  \texttt{dravyans@cs.cmu.edu} }
\date{}

\begin{document}

\maketitle

\begin{abstract}
    We consider a novel data driven approach for designing learning algorithms that can effectively learn with only a small number of labeled examples. This is crucial for modern machine learning applications where labels are scarce or expensive to obtain. We focus on graph-based techniques, where the unlabeled examples are connected in a graph  under the implicit assumption that similar nodes likely have similar labels. Over the past decades, several elegant graph-based semi-supervised learning algorithms for how to infer the labels of the unlabeled examples given the graph and a few labeled examples have been proposed. However, the problem of  how to create the graph (which impacts the practical usefulness of these methods significantly) has been relegated to domain-specific art and heuristics and no general principles have been proposed. In this work we present a  novel data driven approach for learning the graph and provide strong formal guarantees in both the distributional and online learning formalizations.

We show how to leverage problem instances coming from an underlying problem domain to learn the graph hyperparameters from commonly used parametric families of graphs that perform well on new instances coming from the same domain. We obtain low regret and efficient algorithms in the online setting, and generalization guarantees in the distributional setting. We also show how to combine several very different similarity metrics and learn multiple  hyperparameters, providing  general techniques to apply to large classes of problems. We expect some of the tools and techniques we develop along the way to be of interest beyond semi-supervised learning, for data driven algorithms for combinatorial problems more generally.

\end{abstract}

\section{Introduction}
In recent years machine learning techniques have found gainful application in diverse settings including textual, visual, or acoustic data. A major bottleneck of the currently used approaches is the heavy dependence on expensive labeled data. At the same time advances in cheap computing and storage have made it relatively easier to store and process large amounts of unlabeled data. Therefore, an important focus of the present research community is to develop general (domain-independent) methods to learn effectively from the unlabeled data, along with a small amount of labels. Achieving this goal would significantly elevate the state-of-the-art machine intelligence, which currently lags behind the human capability of learning from a few labeled examples. Our work is a step in this direction, and provides algorithms and guarantees that enable fundamental techniques for learning from limited labeled data to provably adapt to problem domains.

Graph-based approaches have been popular for learning from unlabeled data for the past two decades \citep{zhu2009introduction}. 
Labeled and unlabeled examples form the graph nodes and (possibly weighted) edges denote the feature similarity between examples. The implicit modeling assumption needed to make semi-supervised learning possible is that the likelihood of having a particular label increases with closeness to nodes of that label (\citet{balcan2010discriminative}). The graph therefore captures how each example is related to other examples, and by optimizing a suitably regularized objective over it one obtains an efficient discriminative, nonparametric method for learning the labels. There are several well-studied ways to define and regularize an objective on the graph \citep{chapelle2006semi,zhu2009introduction}, and all yield comparable results which strongly depend on the graph used. A general formulation is described as follows, variations on which are briefly discussed under related work.

{\bf Problem formulation}: Given sets $L$ and $U$ of labeled and unlabeled examples respectively, and a similarity metric $d$ over the data, the goal is to use $d$ to extrapolate labels in $L$ to $U$. A graph $G$ is constructed with $L+U$ as the nodes and weighted edges $W$ with $w(u,v)=g(d(u,v))$ for some $g:\R_{\ge 0}\rightarrow\R_{\ge 0}$. We seek labels $f(\cdot)$ for nodes $u$ of $G$ which minimize a regularized loss function $l(f)=\alpha \sum_{v\in L} \hat{l}(f(v),y_v)+\beta H(f,W)+\gamma \norm{f}^2$, under some constraints on $f$. The objective $H$ captures the {\it smoothness} (regularization) induced by the graph (see Table \ref{table:algos} for examples) and $\hat{l}(f(v),y_v)$ is the misclassification loss (computed here on labeled examples).

The graph $G$ takes a central position in this formulation. However, the majority of the research effort on this problem has focused on how to design and optimize the regularized loss function $l(f)$, the effectiveness of which crucially depends on $G$. Indeed the graph $G$ is expected to reflect a deep understanding of the problem structure and how the unlabeled data is expected to help. Despite the central role of $G$ in the semi-supervised learning process, only some heuristics are known for setting the graph hyperparameters \citep{zhu2005semi}. There is no known principled study on how to do this and prior work largely treats this as a domain-specific art. 
Is it possible to acquire the required domain expertise, without involving human experts?

In this work we provide an affirmative answer by introducing data-driven algorithms for building the graphs, that is techniques which learn a provably good problem-specific graph from instances of a learning problem. More precisely, we are required to solve not only one instance of the problem, but multiple instances of the underlying algorithmic problem that come from the same domain \citep{balcan2020data,balcan2017learning,gupta2017pac}. This approach allows us to model the problem of identifying a good algorithm from data as an online or statistical learning problem. 
We formulate the problem of creating the learning graph as a data-specific decision problem, where we select the graph from well-known infinite families of candidates which capture a range of ways to encode example similarity. We show learning a near-optimal graph over these families is possible in both online and distributional settings. In the process we generalize and extend results developed in the context of other data-driven learning problems, and obtain practical methods to build the graphs with strong guarantees. In particular, we show that the approach may be used for learning several parameters at once, and it is useful for learning a broader class of parameters than previously known.

\subsection{Our results}

{\bf Semi-supervised learning}: We consider common ways of creating the graph $G$ and formulating the regularized loss $l(f)$ as families of algorithms to learn over, and by learning over these families determine the most suitable semi-supervised learning method for any given problem domain. The graph is created by setting the edge weights according to some (monotonic) function $g$ of the similarity metric $d$, and is parameterized by $\rho$. We denote this graph by $G(\rho)$ (omitting $d$ for conciseness). Note that each value of $\rho$ corresponds to a semi-supervised learning algorithm, where labels for the unlabeled examples are predicted by optimizing over the graph $G(\rho)$ (i.e., similar nodes according to $G(\rho)$ get similar labels). We consider online and distributional settings and provide efficient algorithms to obtain low regret and low error respectively for learning $\rho$.

In the online setting, we receive problem instances $L_i,U_i$ sequentially and must predict labels $f_i$ for the $i$th instance before receiving the next by optimizing over some graph $G(\rho_i)$. We also observe the loss $l(f)$ for prediction according to $G(\rho)$ for all $\rho$ ({\it full information setting}) or some interval containing $\rho_i$ after our prediction ({\it semi-bandit setting}). The performance is measured by the regret of our predictions using $G(\rho_i)$ relative to the optimal graph $G(\rho^*)$ in hindsight. Our key insight is to note that the loss is a piecewise constant function of the parameter $\rho$ with {\it dispersed} discontinuities (Definition \ref{def:dis}) under mild smoothness assumptions on the similarity function, that the metric is not exact and small perturbations to the similarities does not affect learning. Roughly speaking, dispersion means that the discontinuities are not concentrated in a small region, across instances. {The full information setting however can be computationally inefficient, since it involves computing the loss for each of potentially prohibitively many constant performance ``pieces''. This is overcome by Algorithm \ref{alg:ddsslsb} in the semi-bandit setting, where it is sufficient to compute the loss for a small number of pieces contained in an efficiently computable {\it feedback set}. Our implementation involves a novel min-cut and flow recomputation algorithm on a graph with continuously changing edge-weights, and may be of independent interest to the broader theory community.}

In the distributional setting, the problem instances are assumed to be sampled according to some underlying distribution, and we would like to show PAC bounds for learning with low error with a high confidence. We provide asymptotically tight upper and lower bounds on the pseudodimension of learning the best parameter from a parameterized family of semi-supervised learning algorithms, each algorithm corresponding to a graph $G(\rho)$. We consider both unweighted and weighted graph families. Our bounds imply efficient algorithms with PAC guarantees for the unweighted setting, and hardness of efficient learning of worst case instances of weighted graphs. For commonly used approaches to create a weighted graph from a similarity metric, we show efficient learning is still possible under the mild smoothness assumptions used  in the online setting above. {The lower bounds are fairly technical and involve constructing a family of graph instances while setting the distances between graph nodes in a precise correlated manner to ensure that the loss as a function of the parameter oscillates highly and at carefully determined points in the domain.}

Compared to known heuristics to build graphs which consider a fixed problem instance, our approach may be viewed as building graphs by learning over subsets of the full dataset (which previous approaches have not considered to the best of our knowledge) and doing learning with these instances as examples for the hyperparameter learning problem. The setting however is more general and captures batches of partially labeled data arriving online or according to some distribution.


{\bf Multiple metrics}:
In practice, we might have several natural, but very different types of metrics for our problem. The Euclidean distance metric $d(u,v)$ over the representation (embedding) of the examples alone may not best capture the similarity measure between node pairs. When learning over multiple channels or modes simultaneously, for example in detecting people or activities from video footages, one needs to combine information from several natural similarity metrics \citep{balcan2005person}. 
We can view this as a graph with multiple hyperparameters $G(\rho_1,\rho_2, \dots)$, where the additional parameters indicate relative importance of the different metrics. We show how to select a provably good interpolation by generalizing results from \cite{dick2020semi} to multiple parameters. We use tools from algebraic geometry including the Tarski–Seidenberg theorem and properties of the cylindrical algebraic decomposition to accomplish this.

{\bf Data-driven algorithm design}: This work employs and extends powerful general techniques and results for selecting algorithms from a parameterized family in a data-driven way i.e. from several problem instances.  {\it Dispersion} is a property of problem sequences (observed online, or drawn from a distribution) which has been shown to be necessary and sufficient for provably learning optimal parameters for combinatorial algorithms \cite{balcan2018dispersion,sharma2020learning}. {Algorithms for learning dispersed instances are known for both full information and semi-bandit online settings \cite{balcan2018dispersion,dick2020semi}. We study data driven algorithm design for a completely new setting, learning the graph for graph based semi-supervised learning techniques, and undertake the technical work needed to understand the underlying structure of the induced loss functions in these new settings.} 
In the process we extend general tools for deducing dispersion for general algorithm design problems. {Firstly, for one dimensional loss functions, we show a novel structural result for proving dispersion when discontinuities (for loss as function of the algorithm parameter) occur along roots of exponential polynomials with random coefficients with bounded joint distributions (previously shown only for algebraic polynomials \cite{dick2020semi}). This is crucial for showing learnability in the Gaussian graph kernels setting. Secondly, 
prior work \cite{dick2020semi} was only able to prove dispersion when the discontinuities occur along algebraic curves with random coefficients in just two dimensions. By a novel algebraic and learning theoretic argument we are able to analyze higher dimensions, and show dispersion for an arbitrary constant number of parameters, making it much more generally applicable.}


{\bf Key challenges}: We present the first theoretically grounded work outlining how to create good graphs for learning from unlabeled data. Graph-based semi-supervised learning literature has largely been focused on learning approaches given a graph and very little progress made on the arguably more significant problem of designing good graphs. The problem was noted by \cite{zhu2005semi} and has remained largely open for two decades. {We use a data-driven algorithm design perspective \citep{balcan2018dispersion,dick2020semi} and take steps towards resolving this problem. We remark that our techniques are very general and they apply simultaneously for learning the graph when we do prediction by optimizing various quadratic objectives with hard or soft labels (Table \ref{table:algos}).} 

Online learning in our setting poses some interesting challenges. The loss function is a non-Lipschitz function of the parameter, so usual gradient-based approaches do not work. We use mild perturbation invariance assumptions to show dispersion of the non-Lipschitzness which is necessary to overcome the worst case lower bounds. Furthermore, most previously studied settings for dispersion involve polynomially many discontinuities, so efficient algorithms are immediate, which may not be the case for our setting. Instead we crucially rely on semi-bandit algorithms to ensure that the parameters may be learned efficiently, which involve development of careful local search techniques in the parameter space. For weighted graphs and combinatorial optimizations, the challenge of computing changing mincuts with continuously varying graph weights requires a novel algorithm with combinatorial and continuous elements. In the distributional setting, we provide lower bounds for the pseudo-dimension of learning the algorithm, which require technical constructed instances. {In particular we seek to show instances with highly oscillating loss functions at carefully determined points as the graph parameter is varied, which is especially difficult for the weighted graph families.} We note that even for single-parameter families the lower bound is superconstant (for example $\Omega(\log n)$ and even $\Omega(n)$).


Another key challenge we overcome is to show how the results may be extended to tuning several graph hyperparameters at the same time, making our results much more powerful. Similar results are known for linkage-based clustering \citep{balcan2019learning} but they crucially rely on the algorithm depending on relative distances and not the actual values,  and therefore do not extend to our setting. The problem is significantly more complex than the one-dimensional problem as non-Lipschitzness now occurs along high-dimensional surfaces instead of just points of discontinuity for which learnability may be deduced by arguing about the concentration of these points in intervals. For algebraic curves (learning two parameters) one may use a bound on the number of local extrema and curve intersections in a fixed direction \citep{dick2020semi}, but we need a careful projection argument and tools from algebraic geometry to generalize to higher dimensions.

{\bf Extension to active learning}: We also consider an active learning setting where we learn the graph for the graph-based active learning procedure for a fixed constant label budget. We consider data-driven construction of the graph family, where the labels used for semi-supervised learning are obtained by the greedy active learning algorithm from \cite{zhu2003combining}. The procedure selects the next node which minimizes the expected estimated risk after querying the node. We show how to learn a graph for which the budgeted active learning procedure, followed by semi-supervised label predictions, results in provably near-minimum loss. 


\subsection{Related work}

\begin{table}
\centering
\begin{tabular}{l|c|c|c}
Algorithm & $(\alpha,\beta,\gamma)$ & $H(f,W), \norm{\cdot}$ & Constraints on $f$   \\
\hline
 &&&\\
 A. Mincut (\citet{blum2001learning})               &       $(\infty,1,0) $               & $f^T(D-W)f$ &    $f\in \{0,1\}^n$                  \\
 &&&\\
 B. Harmonic functions (\citet{zhu2003semi})          &   $(\infty,1,0) $                    &   $f^T(D-W)f$ &      $f\in [0,1]^n$                \\
 &&&\\
  C. Normalized cut (\citet{shi2000normalized})              &         $(\infty,1,0) $              & $f^T(D-W)f$  &       $f^T\mathbf{1}=0,f^Tf=n^2,$             \\
  &&&$f\in [0,1]^n$\\
  D. Label propagation (\citet{zhou2004learning})            &           $(1,\mu,1) $              & $f^T\mathcal{L} f$, $\norm{\cdot}_2$  &     $f\in [0,1]^n$                 \\
 &&&\\
\end{tabular}
\vspace*{0.4cm}

\begin{caption}
\\Optimization using a quadratic objective involved in some prominent algorithms for graph-based semi-supervised learning. Here $D_{ij}:=\bI[i=j]\sum_{k}W_{ik}, \mathcal{L}:=D^{-1/2}(D-W)D^{-1/2}$ and the objective is $l(f)=\alpha \sum_{u\in L} (f(u)-y_u)^2+\beta H(f,W)+\gamma \norm{f}^2$.\label{table:algos}
\end{caption}
\end{table}

{\it Semi-supervised learning} is a paradigm for learning from both labeled and unlabeled data (\citet{zhu2009introduction}). It resembles human learning behavior more closely than fully supervised and fully unsupervised models (\citet{gibson2013human,zhu2007humans}). Learning from unlabeled data may be possible due to implicit {\it compatibility} between the target concept and the underlying data distribution (\citet{balcan2010discriminative}). A semi-supervised learning method makes a particular compatibility assumption and provides a procedure to determine a predictor which fits the labeled data well and has high compatibility.

{\it Graph-based methods}: A prominent and effective approach widely used in practice for semi-supervised learning is to optimize a graph-based objective. Here the compatibility assumption is that the labels are {\it smooth} over the graph, and as such the performance is highly sensitive to the graph structure and the edge weights. Since labels partition the graph, we seek a (possibly soft) graph cut as the predictor. Several methods have been proposed to obtain predictors given a graph including $st$-mincuts (\citet{blum2001learning}), soft mincuts that optimize a quadratic energy objective (\citet{zhu2003semi}), label propagation (\citet{zhu2002learning}), and many more (\citet{blum2004semi,belkin2006manifold}). Table \ref{table:algos} summarizes the optimization involved in some prominent algorithms. $\alpha=\infty$ corresponds to forcing labels of labeled examples $L$.

However, it is not clear how to create the graph itself on which the extensive literature stands, although some heuristics are known (\citet{zhu2005semi}).
\citet{zemel2004proximity} discuss how to create a robust graph by considering an ensemble of minimum spanning trees for several data perturbations and randomly retaining edges which appear often. The algorithm however uses a parameter $t$ for expected graph density and it is unclear how to set it for any given problem instance, and no theoretical guarantees are provided. \citet{sindhwani2005beyond} construct {\it warped} kernels more aligned with the data geometry, but the performance may vary strongly with warping and it is not clear how to optimize over it. To the best of our knowledge, we provide the first techniques that yield provably near-optimal graphs. Prior art largely compares semi-supervised graph learning in terms of assumption generality and experimental evidence.

{\it Data-driven design and dispersion}: \citet{gupta2017pac} define a formal learning framework for selecting algorithms from a family of heuristics or a range of hyperparameters. The framework is further developed by \citet{balcan2017learning} and its usefulness as a fundamental algorithm design perspective has been noted \citep{balcan2020data,blum2020technical}. It has been successfully applied to several combinatorial problems like integer programming and clustering \citep{balcan2018learning,balcan2019learning,balcan2018data} and for giving powerful guarantees like adversarial robustness, adaptive learning and differential privacy \citep{balcan2018dispersion,balcan2020power,sharma2020learning,vitercik2019estimating}.  \citet{balcan2020data} provides a simple introduction to and a comprehensive survey on this rapidly expanding research direction.

\citet{balcan2018dispersion} introduce {\it dispersion}, a useful property of the problem instances with respect to an algorithm family, which, if satisfied, intuitively allows it to be efficiently learned in {\it full information} online as well as distributional settings. Full information may be expensive to compute and work with, and a semi-bandit algorithm is introduced by \citet{dick2020semi}. The same work also presents a general technique for analyzing dispersion which is used to show dispersion when the non-Lipschitzness occurs along roots of polynomials with random coefficients, for up to two parameters. {We apply dispersion in a new problem setting, and show how to learn algorithms for semi-supervised learning of labels by carefully studying the properties of these problems.} We also extend the technique, and prove that dispersion holds for a broader setting involving non-polynomial discontinuities and employ tools from algebraic geometry to extend the theory to an arbitrary constant number of parameters.

\section{Notation and definitions}\label{sec:notation}
We are given some labeled points $L$ and unlabeled points $U$. One constructs a graph $G$ by placing (possibly weighted) edges $w(u,v)$ between pairs of data points $u,v$ which are `similar', and labels for the unlabeled examples are obtained by optimizing some graph-based score. We have an oracle $O$ which on querying provides us the labeled and unlabeled examples, and we need to pick $G$ from some family $\G$ of graphs. We commit to using some algorithm $A(G,L,U)$ (abbreviated as $A_{G,L,U}$) which provides labels for examples in $U$, and we should pick a $G$ such that $A(G,L,U)$ results in small error in its predictions on $U$. To summarize more formally,

{\it Problem statement}: Given data space $\X$, label space $\Y$ and an oracle $O$ which yields a number of labeled examples  $L\subset \X\times\Y$ and some unlabeled examples $U\subset \X$ such that $|L|+|U|=n$. We are further given a parameterized family of graph construction procedures over parameter space $\cP$, $\G:\cP\rightarrow(\X\times\X\rightarrow \R_{\ge 0})$, graph labeling algorithm $A:(\X\times\X\rightarrow \R_{\ge 0})\times 2^{\X\times(\Y\cup \{\perp\})}\rightarrow (\X\rightarrow\Y)$, a loss function $l:\Y\times\Y\rightarrow [0,1]$ and a target labeling $\tau:U\rightarrow \Y$. We need to select $\rho\in \cP$ such that corresponding graph $G(\rho)$ minimizes $\sum_Ul(A_{G(\rho),L,U}(u),\tau(u))$ w.r.t. $\rho$.

We will consider online and distributional settings of the above problem. In the online setting we make no distributional assumptions about the data and simply seek to minimize the regret, i.e. the loss suffered in an arbitrary online sequence of oracle queries $O$ relative to that endured by the best parameter $\rho^*$ in hindsight. In the distributional setting we will assume that the data and labels supplied by $O$ come from an underlying distribution $\D$ and we would like to minimize the expected loss suffered on test examples drawn from the distribution with high probability. We will present further details and notations for the respective settings in the subsequent sections.

We will now describe graph families $\G$ and algorithms $A_{G,L,U}$ considered in this work. We assume there is a feature based {\it similarity function} $d:\X\times\X\rightarrow\R_{\ge 0}$, a metric which monotonically captures similarity between the examples. In section \ref{sec:multiple-metrics}, we will consider creating graphs with several similarity functions, but for now assume we have a single $d$. Definition \ref{defn:g} summarizes commonly used parametric methods to build a graph using the similarity function.



In this work, we will consider three parametric families of graph construction algorithms defined below. $\bI[\cdot]$ is the indicator function taking values in $\{0,1\}$.

\begin{defn}\label{defn:g} Graph kernels.
\begin{enumerate}
   \item[a)] Threshold graph, $G(r)$. Parameterized by a threshold $r$, we set $w(u,v)=\bI[d(u,v)\le r]$.
    \item[b)] Polynomial kernel, $G(\Tilde{\alpha})$. $w(u,v)=(\Tilde{d}(u,v)+\Tilde{\alpha})^d$ for fixed degree $d$, parameterized by $\Tilde{\alpha}$.\footnote{With some notational abuse here, we have $d$ as the integer degree of the polynomial, and $\Tilde{d}(\cdot,\cdot)$ as the similarity function.}
    \item[c)] Gaussian RBF or exponential kernel, $G(\sigma)$. $w(u,v)=e^{-d(u,v)^2/\sigma^2}$, parameterized by $\sigma$.
\end{enumerate}

\end{defn} 

\begin{remark}
Another popular family of graphs used in practice is the $k$ nearest neighbor graphs, where $k\in\{0,1,\dots,n-1\}$, $n$ is the number of nodes in the graph, is the parameter. Even though $k$-NN graphs may result in different graphs the ones considered in the paper, learning how to build an optimal graph over the algorithm family $G(k)$ is much simpler. Online learning of the parameter $k$ in this setting can be recognized as an instance of learning with experts advice for a finite hypothesis class (Section 3.1 of \cite{shalev2011online}), where an upper bound of $O(\sqrt{T\log n})$ is known for the Weighted Majority algorithm. Online-to-batch conversion provides generalization guarantees in the distributional setting (Section 5 of \cite{shalev2011online}). We remark that our algorithm families need more sophisticated analysis due to continuous ranges of the algorithm parameters.
\end{remark}

The threshold graph adds (unweighted) edges to $G$ only when the examples are closer than some $r\in\R_{\ge 0}$, i.e. a step function of the distance. Polynomial and exponential kernels add (weighted) edges to the graph, with weights varying polynomially and exponentially (respectively) with the similarity. Note that similarity function $\Tilde{d}(u,v)$ in the definition for polynomial kernels  increases monotonically with similarity of examples, as opposed to the other two\footnote{Common choices are setting $d(u,v)$ as the Euclidean norm and $\Tilde{d}(u,v)$ as the dot product when $u,v\in\R^n$}. Usually the threshold graph setting (Definition \ref{defn:g}a) will be easier to optimize over, but it is also a small parameter family  often with relatively weaker performance in practice. In the following, we will refer to this setting by the {\it unweighted graph} setting, and the other two settings (Definitions \ref{defn:g}b and \ref{defn:g}c) by the {\it weighted graph} setting. Often we will discuss just the Gaussian RBF setting since it is more technically challenging and more commonly used in practice for building graphs. However, in some instances working through the polynomial kernel setting can provide useful insights. 

Once the graph is constructed using one of the above kernels, we can assign labels using a suitable algorithm $A_{G,L,U}$. A popular and effective approach is by optimizing a quadratic objective $\frac{1}{2}\sum_{u,v}w(u,v)(f(u)-f(v))^2=f^T(D-W)f$. Here $f$ may either be discrete $f(u)\in\{0,1\}$ which corresponds to finding a graph mincut separating the oppositely labeled vertices \citep{blum2001learning}, or $f$ may be continuous, i.e. $f\in[0,1]$, and we can round $f$ to obtain the labels \citep{zhu2003semi}. These correspond to algorithms A and B respectively from Table \ref{table:algos}. It is noted that all algorithms have comparable performance provided the graph $G$ encodes the problem well \citep{zhu2009introduction}. We restrict our attention to these two algorithms for simplicity of presentation, although our algorithms and proofs may be extended to any quadratic objective based algorithm in Table \ref{table:algos}\footnote{Specifically by extending arguments for algorithm B since the optimization is similar. In contrast, Algorithm A is combinatorial and the reasoning diverges somewhat.}.

Finally we note definitions of some useful learning theoretic complexity measures. First recall the definitions of pseudodimension and Rademacher complexity, well-known measures for hypothesis-space complexity in statistical learning theory. Bounding these quantities implies immediate bounds on learning error using classic learning theoretic results. In  Section \ref{sec: distrib} we will bound the pseudodimension and Rademacher complexity for the problems of learning unweighted and weighted graphs.
\begin{defn}Pseudo-dimension \citep{pollard2012convergence}. Let $\h$ be a set of real valued functions from input space $\X$. We say that
$C = (x_1, \dots, x_m)\in \X^m$ is pseudo-shattered by $\h$ if there exists a vector
$r = (r_1, \dots, r_m)\in\R^m$ (called ``witness”) such that for all
$b= (b_1, \dots, b_m)\in \{\pm 1\}^m $ there exists $h_b\in \h$ such that $\text{sign}(h_b(x_i)-r_i)=b_i$. Pseudo-dimension of $\h$  is the cardinality of the largest set
pseudo-shattered by $\h$.
\end{defn}

\begin{defn}
Rademacher complexity \citep{bartlett2002rademacher}. Let $\F=\{f_\rho:\X\rightarrow [0,1], \rho\in\C\subset\R^d\}$ be a parameterized family of functions, and sample $\s=\{x_i,\dots,x_T\}\subseteq\X$. The empirical Rademacher complexity of $\F$ with respect to $\s$ is defined as $\hat{R}(\F,\s)=\E_{\mathbf{\sigma}}\left[\sup_{f\in\F}\frac{1}{T}\sum_{i=1}^T\sigma_if(x_i)\right]$, where $\sigma_i\sim U(\{-1,1\})$ are Rademacher variables.
\end{defn}

We will also need the definition of {\it dispersion} which, informally speaking, captures how amenable a non-Lipschitz function is to online learning. As noted in \citep{balcan2018dispersion,sharma2020learning}, dispersion is necessary and sufficient for learning piecewise Lipschitz functions.
\begin{defn}\label{def:dis} Dispersion \citep{dick2020semi}.
The sequence of random loss functions $l_1, \dots,l_T$ is $\beta$-{\it dispersed} for the Lipschitz constant $L$ if, for all $T$ and for all $\epsilon\ge T^{-\beta}$, we have that, in expectation, at most
$\Tilde{O}(\epsilon T)$ functions (the soft-O notation suppresses dependence on quantities beside $\epsilon,T$ and $\beta$, as well as logarithmic terms)
are not $L$-Lipschitz for any pair of points at distance $\epsilon$ in the domain $\C$. That is, for all $T$ and for all $\epsilon\ge T^{-\beta}$,
\begin{align*}
    \E\left[
\max_{\substack{\rho,\rho'\in\C\\\norm{\rho-\rho'}_2\le\epsilon}}\big\lvert
\{ t\in[T] \mid l_t(\rho)-l_t(\rho')>L\norm{\rho-\rho'}_2\} \big\rvert \right] 
\le  \Tilde{O}(\epsilon T).
\end{align*}
\end{defn}

\section{New general dispersion-based tools for data-driven design}\label{sec:dtools}

We present new techniques and generalize known tools for analyzing data-driven algorithms \cite{balcan2018dispersion,dick2020semi}. Our new tools apply to a very broad class of algorithm design problems, for which we derive sufficient {\it smoothness} conditions to infer dispersion of a random sequence of problems, i.e. the algorithmic performance as a function of the algorithm parameters is dispersed. \citet{dick2020semi} provide a general tool for verifying dispersion if non-Lipschitzness occurs along roots of (algebraic) polynomials in one and two dimensions. We improve the results in the following two ways.

Our first result is that dispersion for one-dimensional loss functions follows when the points of discontinuity occur at the roots of exponential polynomials if the coefficients are random, lie within a finite range, and are drawn according to a bounded joint distribution. In addition to generalizing prior results, we present a new simpler proof. The full proof appears in Appendix \ref{app:exp}.

\begin{theorem}\label{thm:exproots1}
Let $\phi(x)=\sum_{i=1}^na_ie^{b_ix}$ be a random function, such that coefficients $a_i$ are real and of magnitude at most $R$, and distributed with joint density at most $\kappa$. Then for any interval $I$ of width at most $\epsilon$, P($\phi$ has a zero in $I$)$\le \Tilde{O}(\epsilon)$ (dependence on $b_i,n,\kappa,R$ suppressed).
\end{theorem}

\begin{proof}[Proof Sketch]
For $n=1$ there are no roots, so assume $n>1$. Suppose $\rho$ is a root of $\phi(x)$. Then $\mathbf{a}=(a_1,\dots,a_n)$ is orthogonal to $\varrho(\rho)=(e^{b_1\rho},\dots,e^{b_n\rho})$ in $\R^n$. For a fixed $\rho$, the set $S_\rho$ of coefficients $\mathbf{a}$ for which $\rho$ is a root of $\phi(y)$ lie along an $n-1$ dimensional linear subspace of $\R^n$. Now $\phi$ has a root in any interval $I$ of length $\epsilon$, exactly when the coefficients lie on $S_\rho$ for some $\rho\in I$.  The desired probability is therefore upper bounded by $\max_{\rho}\textsc{Vol}(\cup S_y\mid y\in [\rho - \epsilon, \rho + \epsilon])/\textsc{Vol}(S_y\mid y\in \R)$ which we will show to be $\Tilde{O}(\epsilon)$. The key idea is that if $|\rho-\rho'|<\epsilon$, then $\varrho(\rho)$ and $\varrho(\rho')$ are within a small angle $\theta_{\rho,\rho'}=\Tilde{O}(\epsilon)$ for small $\epsilon$ (the probability bound is vacuous for large $\epsilon$). But any point in $S_{\rho}$ is at most $\Tilde{O}(\theta_{\rho,\rho'})$ from a point in $S_{\rho'}$, which implies the desired bound.
\end{proof}

We further go beyond single-parameter discontinuties, which occur as points along a line to general small dimensional parameter spaces $\R^p$, where discontinuties can occur along algebraic hypersurfaces.
We employ tools from algebraic geometry to establish a bound on shattering of algebraic hypersurfaces by axis-aligned paths (Theorem \ref{thm:alg-hyp}), which implies dispersion using a VC dimension based argument (Theorem \ref{thm:VC-bound-general}). Our result is the first of its kind, a general sufficient condition for dispersion for any constant number $p$ of parameters, and applies to a broad class of algorithm families. Full proofs may be found in Appendix \ref{app:multimetric}.

\begin{restatable}{theorem}{thmalghyp}\label{thm:alg-hyp}
There is a constant $k$ depending only on $d$ and $p$ such that axis-aligned line segments in $\R^p$ cannot
shatter any collection of $k$ algebraic hypersurfaces of degree at most $d$.
\end{restatable}
\begin{proof}[Proof Sketch]
Let $\C$ denote a collection of $k$ algebraic hypersurfaces of degree at most $d$ in $\R^p$. 
We say that a subset of $\C$ is {\it hit} by a line segment if the subset is exactly the set of curves in $\C$ which intersect the segment, and {\it hit} by a line if some segment of the line hits the subset. We can upper bound the subsets of $\C$ by line segments in a fixed axial direction $x$ in two steps. Along a fixed line, Bezout's theorem bounds the number of intersections and therefore subsets hit by different line segments. The lines along $x$ can further be shown to belong to equivalence classes corresponding to cells in the cylindrical algebraic decomposition of the projection of the hypersurfaces, orthogonal to $x$. Finally, we can extend this to axis-aligned segments by noting they may hit only $p$ times as many subsets.
\end{proof}


\begin{restatable}{theorem}{thmvcgeneral}\label{thm:VC-bound-general}
 Let $l_1, \dots, l_T : \R^p \rightarrow \R$ be independent piecewise $\mathcal{L}$-Lipschitz functions, each having discontinuities specified by a collection of at most $K$ algebraic hypersurfaces of bounded degree. Let $L$ denote the set of axis-aligned paths between pairs of points in $\R^p$, and for each $s\in L$ define
 $D(T, s) = |\{1 \le t \le T \mid l_t\text{ has a discontinuity along }s\}|$. Then we
have $\E[\sup_{s\in L} D(T, s)] \le \sup_{s\in L} \E[D(T, s)] +
O(\sqrt{T \log(TK)})$.
\end{restatable}
\begin{proof}[Proof Sketch]
We  relate the
number of ways line segments can label vectors of $K$ algebraic hypersurfaces of degree $d$ to the VC-dimension of line
segments (when labeling algebraic hypersurfaces), which from Theorem \ref{thm:alg-hyp} is constant. To verify dispersion,
we need a uniform-convergence bound on the number of Lipschitz failures between the worst pair of points $\rho,\rho'$
at
distance $\le \epsilon$, but the definition allows us to bound the worst rate of discontinuties along any path between $\rho,\rho'$ of our
choice. We can bound the VC dimension of axis aligned segments against bounded-degree algebraic
hypersurfaces, which will allow us to establish dispersion by considering piecewise axis-aligned paths between points $\rho$ and $\rho'$.
\end{proof}

\section{Data-driven semi-supervised learning}
We will warm up this section with a simple example demonstrating the need for and challenges posed by the problem of learning how to build a good graph from data. We will then consider online and distributional settings in sections \ref{sec: ol} and \ref{sec: distrib} respectively. For online learning we show how to employ and extend dispersion based analysis to our setting, and obtain algorithms which learn good graphs with low regret and are efficient to implement, under mild assumptions on data niceness. For the distributional setting, we analyze the pseudodimension and Rademacher complexity of our learning problems which imply generalization guarantees for learning the graph parameters.

{\it Transductive and inductive aspects}: For semi-supervised learning, we may distinguish the transductive setting where predictions are evaluated only on provided unlabeled examples $U$, with the inductive setting where we also care about new unseen examples coming from the same distribution. Graph-based methods were originally introduced as transductive learning approaches \citep{blum2001learning,zhu2002learning}, but may be used in either setting. For induction we may recompute the graph, or use a fixed subgraph (and assume that new points do not affect the transductive labels) for more efficient prediction \citep{delalleau2005efficient}. Our setting has an inductive aspect since we learn a graph (by learning graph parameter values) which we expect to use for unseen problem instances. 
\subsection{Any threshold may be optimal}
We consider the setting of learning thresholds for unweighted graphs (Definition \ref{defn:g}a). We give a simple demonstration that in a {\it single instance} any threshold may be optimal for labelings consistent with graph smoothness assumptions, therefore providing motivation for the learning in our setting. The example below captures the intuition that any unlabeled point may get weakly connected to examples from one class for a small threshold but may get strongly connected to another class as the threshold is increased to a larger value. Therefore depending on the unknown true label either threshold may be optimal or suboptimal, and it makes sense to learn the correct value through repeated problem instances.
\begin{theorem}
Let $r_{\min}$ denote the smallest value of threshold $r$ for which every unlabeled node of $G(r)$ is reachable from some labeled node, and $r_{\max}$ be the smallest value of threshold $r$ for which $G(r)$ is the complete graph. There exists a data instance $(L,U)$ such that for any $r_{\zeta}=\zeta r_{\min}+(1-\zeta) r_{\max}$ for $\zeta\in(0,1)$, there exists a set of labelings $\U$ of the unlabeled points such that for some $U_{\zeta},\Bar{U}_\zeta\in\U$, $r_{\zeta}$ minimizes $l_{A(G(r),L,U_{\zeta})}$ but not  $l_{A(G(r),L,\Bar{U}_\zeta)}$. 
\end{theorem}

\begin{figure}[h]
    \centering

\tikzset{every picture/.style={line width=0.75pt}} 

\begin{tikzpicture}[x=0.75pt,y=0.75pt,yscale=-1,xscale=1]

\draw   (404,301) .. controls (404,292.99) and (410.49,286.5) .. (418.5,286.5) .. controls (426.51,286.5) and (433,292.99) .. (433,301) .. controls (433,309.01) and (426.51,315.5) .. (418.5,315.5) .. controls (410.49,315.5) and (404,309.01) .. (404,301) -- cycle ;
\draw   (569,299.63) .. controls (569,288.37) and (578.12,279.25) .. (589.38,279.25) .. controls (600.63,279.25) and (609.75,288.37) .. (609.75,299.63) .. controls (609.75,310.88) and (600.63,320) .. (589.38,320) .. controls (578.12,320) and (569,310.88) .. (569,299.63) -- cycle ;
\draw   (483,301) .. controls (483,299.34) and (484.34,298) .. (486,298) .. controls (487.66,298) and (489,299.34) .. (489,301) .. controls (489,302.66) and (487.66,304) .. (486,304) .. controls (484.34,304) and (483,302.66) .. (483,301) -- cycle ;
\draw    (418,329) -- (487,329) ;
\draw [shift={(489,328.25)}, rotate = 539.8] [color={rgb, 255:red, 0; green, 0; blue, 0 }  ][line width=0.75]    (10.93,-3.29) .. controls (6.95,-1.4) and (3.31,-0.3) .. (0,0) .. controls (3.31,0.3) and (6.95,1.4) .. (10.93,3.29)   ;
\draw [shift={(416,328.5)}, rotate = 359.8] [color={rgb, 255:red, 0; green, 0; blue, 0 }  ][line width=0.75]    (10.93,-3.29) .. controls (6.95,-1.4) and (3.31,-0.3) .. (0,0) .. controls (3.31,0.3) and (6.95,1.4) .. (10.93,3.29)   ;
\draw    (491,329) -- (593,329) ;
\draw [shift={(595,329.25)}, rotate = 180.54] [color={rgb, 255:red, 0; green, 0; blue, 0 }  ][line width=0.75]    (10.93,-3.29) .. controls (6.95,-1.4) and (3.31,-0.3) .. (0,0) .. controls (3.31,0.3) and (6.95,1.4) .. (10.93,3.29)   ;
\draw [shift={(489,328.25)}, rotate = 0.54] [color={rgb, 255:red, 0; green, 0; blue, 0 }  ][line width=0.75]    (10.93,-3.29) .. controls (6.95,-1.4) and (3.31,-0.3) .. (0,0) .. controls (3.31,0.3) and (6.95,1.4) .. (10.93,3.29)   ;
\draw    (421,265.25) -- (590,265.25) ;
\draw [shift={(592,265.25)}, rotate = 180] [color={rgb, 255:red, 0; green, 0; blue, 0 }  ][line width=0.75]    (10.93,-3.29) .. controls (6.95,-1.4) and (3.31,-0.3) .. (0,0) .. controls (3.31,0.3) and (6.95,1.4) .. (10.93,3.29)   ;
\draw [shift={(419,265.25)}, rotate = 0] [color={rgb, 255:red, 0; green, 0; blue, 0 }  ][line width=0.75]    (10.93,-3.29) .. controls (6.95,-1.4) and (3.31,-0.3) .. (0,0) .. controls (3.31,0.3) and (6.95,1.4) .. (10.93,3.29)   ;
\draw [line width=3]    (433,301) -- (483,301) ;

\draw (412,293) node [anchor=north west][inner sep=0.75pt]   [align=left] {$\displaystyle L_{1}$};
\draw (582,293) node [anchor=north west][inner sep=0.75pt]   [align=left] {$\displaystyle L_{2}$};
\draw (482,286) node [anchor=north west][inner sep=0.75pt]   [align=left] {$\displaystyle a$};
\draw (418,331.5) node [anchor=north west][inner sep=0.75pt]   [align=left] {$\displaystyle r_{\min} ={\textstyle \frac{r^{*}}{2}}$};
\draw (537,331) node [anchor=north west][inner sep=0.75pt]   [align=left] {$\displaystyle r^{*}$};
\draw (487,241.5) node [anchor=north west][inner sep=0.75pt]   [align=left] {$\displaystyle r_{\max}$};

\end{tikzpicture}
    \caption{$G(r)$ connects $a$ to nodes in $L_1$ for $r_{\min}\le r<r^*$.}
    \label{fig:any_threshold}
\end{figure}
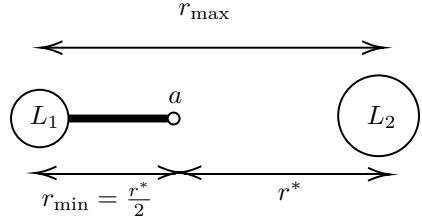
\begin{proof}
Note that for any $r<r_{\min}$, there is no graph similarity information for at least one node, and therefore all labels cannot be predicted. Also, the graph is unchanged for all $r\ge r_{\max}$. Therefore, $r\in[r_{\min},r_{\max}]$ captures all graphs of interest on a given data instance.

Intuitively the statement claims that any threshold $r$ (modulo the scaling factors for the data embedding) may be optimal or suboptimal for some data labeling for a given constructed instance. Therefore it is useful to consider several problem instances and learn the optimal value of $r$ for the data distribution.
We will present an example where an unlabeled point is closest to some labeled point of one class but closer to more points of another class on average. So for small thresholds it may be labeled as the first class and for larger thresholds as the second class.

Let $L=L_1\cup L_2$ with $|L_1|<|L_2|$ and $d(u,v)=0$ for $u,v\in L_i,i\in\{1,2\}$, $d(u,v)=3r^*/2$ for $u\in L_i,v\in L_j, i\ne j$, where $r^*$ is a positive real. Further let $U=\{a\}$ such that $d(a,u_i)=ir^*/2$ for each $u_i\in L_i$. It is straightforward to verify that the triangle inequality is satisfied. Further note that $r_{\min}=r^*/2$ and $r_{\max}=3r^*/2$. Our set of labelings $\U$ will include one that labels $a$ according to each class. Now we have two cases
\begin{enumerate}
    \item $\zeta\in(0,\frac{1}{2})$: $r_{\min}\le r<r^*$, $G(r_{\zeta})$ connects $a$ to $L_1$ but not $L_2$ and we have that the loss is minimized exactly for the labeling where $a$ matches $L_1$.
    \item $\zeta\in[\frac{1}{2},1)$: $ r^*\le r\le r_{\max}$, $G(r_{\zeta})$ connects $a$ to  both $L_1$ and $L_2$. But since $|L_1|<|L_2|$, we predict that the label of $a$ matches that of $L_2$.
\end{enumerate}
Finally we note that $d(u,v)$ may not be exactly zero when $u\ne v$ for a metric. This is easily fixed by making tiny perturbations to the labeled points, for any given $r_{\zeta}$.
\end{proof}

The example presented above captures some essential challenges of our setting in the following sense. Firstly, we see that the loss function may be non-Lipschitz (as a function of the parameter $r$), which makes the optimization problem more challenging. More importantly, it highlights that graph similarity only approximately corresponds to label similarity, and how the accuracy of this correspondence is strongly influenced by the graph parameters. In this sense, it may not be possible to learn from a single instance, and considering a data-driven setting is crucial. 

\subsection{Dispersion and online learning}\label{sec: ol}
We consider the problem of learning the graph online. 
In the online setting, we are presented with instances of the problem and want to learn the best value of the parameter $\rho$ while making predictions. We also assume we get all the labels for past instances which may be used to determine the loss for any $\rho$ ({\it full information})\footnote{We can think of each problem instance to be of a small size, so we do not need too many labels if we can learn with a reasonable number of problem instances. We improve on the label requirement further in the semi-bandit setting.}. A choice of $\rho$ uniquely determines the graph $G(\rho)$ (for example in single parameter families in Definition \ref{defn:g}) and we use some algorithm $A_{G(\rho),L,U}$ to make predictions (e.g. minimizing the quadratic penalty score above) and suffer loss $l_{A(G(\rho),L,U)}:=\sum_{u\in U}l(A_{G(\rho),L,U}(u),\tau(u))$ which we seek to minimize relative to the best fixed choice of $\rho$ in hindsight. Formally, at time $t\in[T]$ we predict $p_t\in\cP$ (the parameter space) based on labeled and unlabeled examples $(L_i,U_i), i\in[t]$ and past labels $\tau(u)$ for each $u\in U_j, j<t$ and seek to minimize
\[R_T:=\sum_{t=1}^Tl_{A(G(\rho_t),L_t,U_t)}-\min_{\rho\in\cP}\sum_{t=1}^Tl_{A(G(\rho),L_t,U_t)}\]

A key difficulty in the online optimization for our settings is that the losses, as noted above, are discontinuous functions of the graph parameters $\rho$. We can efficiently solve this problem if we can show that the loss functions are dispersed, in fact $\frac{1}{2}$-dispersed functions may be learned with $\Tilde{O}(\sqrt{T})$ regret (\cite{balcan2018dispersion,sharma2020learning}). {Algorithm \ref{alg:ddssl} adapts the general algorithm of \cite{balcan2018dispersion} to data-driven graph-based learning and achieves low regret for dispersed functions.}
Recall that dispersion roughly says that the discontinuities in the loss function are not too concentrated. We will exploit an assumption that the embeddings are approximate, so small random perturbations to the distance metric will likely not affect learning. This mild distributional assumption allows us to show dispersion, and therefore learnability. For further background and additional proofs and details from this section, see Appendix \ref{app:ol}.

\begin{algorithm}[t]
\caption{{Data-driven Graph-based SSL ($\lambda$)}}
\label{alg:ddssl}
\begin{algorithmic}[1]
\STATE {\bfseries Input:} Graphs $G_t$ with labeled and unlabeled nodes $(L_t,U_t)$ and node similarities $d(u,v)_{u,v\in L_t\cup U_t}$.
\STATE {\bfseries Hyperparameter:} step size parameter  $\lambda \in (0, 1]$.
\STATE {\bfseries Output:} Graph parameter $\rho_t$ for times $t=1,2,\dots,T$.
\STATE{Set $w_1(\rho)=1$ for all $\rho\in \R_{\ge 0}$.}
\FOR{$t=1,2,\dots,T$}
\STATE{$W_t:=\int_{C}w_t(\rho)d\rho$.}
\STATE{Sample $\rho$ with probability
        $p_{t}(\rho)=\frac{w_t(\rho)}{W_t}$, output as $\rho_t$.}
\STATE{Compute average loss function $l_t(\rho)=\frac{1}{|U_t|}\sum_{u\in U}l(A_{G_t(\rho),L_t,U_t}(u),\tau(u))$.}
\STATE{Set $u_t(\rho)=1-l_t(\rho)$ (loss is converted to utility function, $u_t(\rho)\in[0,1]$).}
\STATE{For each $\rho\in C, \text{ set }w_{t+1}(\rho)=e^{\lambda u_t(\rho)}w_{t}(\rho)$.}
\ENDFOR
\end{algorithmic}
\end{algorithm}

\subsubsection{Dispersion of the loss functions}\label{sec:ol-d}
We start with showing dispersion for the unweighted graph family, with threshold parameter $r$. Here dispersion follows from a simple assumption that the distance $d(u,v)$ for any pair of nodes $u,v$ follows a $\kappa$-bounded distribution\footnote{A density function $f : \R \rightarrow \R$ corresponds to a $\kappa$-bounded
distribution if $\max_{x\in\R}\{f(x)\} \le \kappa$. For example, $\mathcal{N}(\mu,\sigma)$ is $\frac{1}{2\pi\sigma}$-bounded for any $\mu$.}, and observing that discontinuities of the loss (as a function of $r$) must lie on the set of distances $d(u,v)$ in the samples (for any optimization algorithm).

\begin{lemma}\label{lem:disc}
Let $\bar{l}(r)=l_{A(G(r),L,U)}$ be the loss function for graph $G(r)$ created using the threshold kernel $w(u,v)=\bI[d(u,v)\le r]$. Then $\bar{l}(r)$ is piecewise constant and any discontinuity occurs at $r^*=d(u,v)$ for some  graph nodes $u,v$.
\end{lemma}
\begin{proof}This essentially follows from the observation that as $r$ is increased, the graph gets a new edge only for some $r^*=d(u,v)$. Therefore no matter what the optimization algorithm is used to predict labels to minimize the loss, the loss is fixed given the graph, and has discontinuities potentially only when new edges are added. 
\end{proof}
We can use it to show the following theorem.

\begin{theorem}\label{thm:gr-dispersion}
Let $l_1,\dots, l_T:\R\rightarrow\R$ denote an independent\footnote{{Note that the problems arriving online are adversarial. The adversary is smoothed \cite{spielman2004smoothed} in the sense it has a distribution which it can choose as long as it has bounded density over the parameters, independent samples are drawn from adversary's distribution.}} sequence of losses as a function of parameter $r$, when the graph is created using a threshold kernel $w(u,v)=\bI[d(u,v)\le r]$ and labeled by applying any algorithm on the graph. 
If $d(u,v)$ follows a $\kappa$-bounded distribution for any $u,v$, the sequence is $\frac{1}{2}$-dispersed, and there is an algorithm (Algorithm \ref{alg:ddssl}) for setting $r$ with regret upper bounded by $\Tilde{O}(\sqrt{T})$.
\end{theorem}

\begin{proof}
Assume a fixed but arbitrary ordering of nodes in each $V_t=L_t\cup U_t$ denoted by $V_t^{(i)},i\in[n]$. Define $d_{i,j}=\{d(u,v)\mid u=V_t^{(i)},v=V_t^{(j)},t\in[T]\}$. Since $d_{i,j}$ is $\kappa$-bounded, the probability that it falls in any interval of length $\epsilon$ is $O(\kappa\epsilon)$. Since different problem instances are independent and using the fact that the VC dimension of intervals is 2, with probability at
least $1-\delta/D$, every interval of width $\epsilon$ contains at most $O\left(\kappa\epsilon T +\sqrt{T\log D/\delta}\right)$ discontinuities from each $d_{i,j}$ (using Lemma \ref{lem:disc}). Now a union bound over the failure modes for $d_{i,j}$ for different $i,j$ gives $O\left(n^2\kappa\epsilon T +n^2\sqrt{T\log n/\delta}\right)$ discontinuities with probability at least $1-\delta$ for any $\epsilon$-interval. Setting $\delta=1/\sqrt{T}$, for each $\epsilon\ge 1/\sqrt{T}$ the maximum number of discontinuities in any $\epsilon$-interval is at most $(1-\delta)O\left(\kappa n^2\sqrt{T\log n \sqrt{T}}\right)+\delta T=\Tilde{O}(\epsilon T)$, in expectation, proving $\frac{1}{2}$-dispersion.
\end{proof}

We can show dispersion for weighted graph kernels as well, but under stronger assumptions than before. We assume that distances $d(u,v)$ are jointly $\kappa$-bounded on a closed and bounded support. The plan is to use results for dispersion analysis from \cite{dick2020semi} (summarized in Appendix \ref{app:disp-recipe}), which implies that roots of a random polynomial are dispersed provided it has finite coefficients distributed jointly with a $\kappa$-bounded distribution. We establish the following for learning polynomial kernels (full proof in Appendix \ref{app:disp-recipe}).

\begin{restatable}{theorem}{thmpoly}\label{thm:dispersion}
Let $l_1,\dots, l_T:\R\rightarrow\R$ denote an independent sequence of losses as a function of parameter $\Tilde{\alpha}$, when the graph is created using a polynomial kernel $w(u,v)=(\Tilde{d}(u,v)+\Tilde{\alpha})^d$ and labeled by optimizing the quadratic objective $\sum_{u,v}w(u,v)(f(u)-f(v))^2$. If $\Tilde{d}(u,v)$ follows a $\kappa$-bounded distribution with a closed and bounded support, the sequence is $\frac{1}{2}$-dispersed, and the regret of Algorithm \ref{alg:ddssl} may be upper bounded by $\Tilde{O}(\sqrt{T})$.
\end{restatable}
\begin{proof}[Proof Sketch] $w(u,v)$ is a polynomial in $\Tilde{\alpha}$ of degree $d$ and the coefficients are $K\kappa$-bounded, where $K$ is a constant that only depends on $d$ and the support of $\Tilde{d}(u,v)$. Consider the harmonic solution of the quadratic objective \citep{zhu2003semi} which is given by $f_{U}=(D_{UU}-W_{UU})^{-1}W_{UL}f_L$. For any $u\in U$, $f(u)=1/2$ is a polynomial equation in $\Tilde{\alpha}$ with degree at most $nd$. Therefore the labeling, and consequently also the loss function, may only change when $\Tilde{\alpha}$ is a root of one of $|U|$ polynomials of degree at most $dn$. The dispersion result is now a simple application of results from \cite{dick2020semi} (noted as Theorems \ref{thm:poly-roots} and \ref{thm:VC-bound} in Appendix \ref{app:disp-recipe}). The regret bound is implied by Theorem 2 of \citep{dick2020semi}.
\end{proof}

\begin{remark}[Extension to exponential kernel]
We can also extend the analysis to obtain similar results when using the exponential kernel $w(u,v)=e^{-||u-v||^2/\sigma^2}$. The results of \cite{dick2020semi} no longer directly apply as the points of discontinuity are no longer roots of polynomials, and we need to analyze points of discontinuities of {\it exponential polynomials}, i.e. $\phi(x)=\sum_{i=1}^ka_ie^{b_ix}$ (See Section \ref{sec:dtools} and Appendix \ref{app:exp}). The number of discontinuities may be exponentially high in this case. Indeed, solving the quadratic objective shows that discontinuities lie at zeros of exponential polynomials  with $k=O(|U|^n)$ terms.
\end{remark}  

\begin{remark}[Extension to local and global classification \cite{zhou2004learning}] Above results can be extended to the classification algorithm used in \cite{zhou2004learning}. The key observation is that the labels are given by a closed-form matrix, $f^*=(I-\alpha D^{-1/2}WD^{1/2})Y$ or $f^*=(D-\alpha W)Y$ (for the two variants considered). For threshold graphs $G(r)$, the regret bound in Theorem \ref{thm:gr-dispersion} applies to any classification algorithm. Extension to polynomial kernels $G(\Tilde{\alpha})$ is described below.
For fixed $\alpha$ (in the notation of \cite{zhou2004learning}, in expression for $f^*$ above), the discontinuities in the loss as a function of the parameter $\Tilde{\alpha}$ lie along roots of polynomials in the parameter $\Tilde{\alpha}$ and therefore the same proof as Theorem \ref{thm:dispersion} applies (essentially we get polynomial equations with slightly different but still $K$-bounded coefficients). On the other hand, if we consider $\alpha$ as another graph parameter, we can still learn the kernel parameter $\Tilde{\alpha}$ together with $\alpha$ by applying Theorem \ref{thm:poly-roots} and Theorem \ref{thm:VC-bound-general} (instead of Theorem \ref{thm:VC-bound}) in the proof of Theorem \ref{thm:dispersion}. 
\end{remark}

\subsubsection{Combining several similarity measures}\label{sec:multiple-metrics}
Often the distance metric used for measuring similarity between the data points is a heuristic, and we can have multiple reasonable metrics.  Different metrics may have their own advantages and issues and often a weighted combination of metrics, say $\sum_i{\rho_i}d_i(\cdot,\cdot)$, works better than any individual metric. This has been observed in practice for semi-supervised learning \citep{balcan2005person}. The combination weights $\rho_i$ are additional graph hyperparameters. A combination of metrics has been shown to boost performance theoretically and empirically for linkage-based clustering \citep{balcan2019learning}. However the argument therein crucially relies on the algorithm depending on relative distances and not the actual values, and therefore does not extend directly to our setting. We develop a first general tool for analyzing dispersion for multi-dimensional parameters (Section \ref{sec:dtools}), which implies the multi-parameter analogue of Theorem \ref{thm:dispersion}, stated as follows. See Appendix \ref{app:multimetric} for proof details.




\begin{restatable}{theorem}{thmmultimetric}\label{thm:mutlimetric}
Let $l_1,\dots, l_T:\R^p\rightarrow\R$ denote an independent sequence of losses as a function of parameters $\rho_i,i\in[p]$, when the graph is created using a polynomial kernel $w(u,v)=(\sum_{i=1}^{p-1}\rho_i\Tilde{d}(u,v)+\rho_p)^d$ and labeled by optimizing the quadratic objective $\sum_{u,v}w(u,v)(f(u)-f(v))^2$. If $\Tilde{d}(u,v)$ follows a $\kappa$-bounded distribution with a closed and bounded support, the sequence is $\frac{1}{2}$-dispersed, and the regret may be upper bounded by $\Tilde{O}(\sqrt{T})$.
\end{restatable}





\subsubsection{Semi-bandit setting and efficient algorithms}\label{sec: semibandit}

Online learning with full information is usually inefficient in practice since it involves computing and working with the entire domain of hyperparameters. For our setting in particular this is computationally infeasible for weighted graphs since the number of pieces (in loss as a piecewise constant function of the parameter) may be exponential in the worst case (see section \ref{sec: pd}). Fortunately we have a workaround provided by \citet{dick2020semi} where dispersion implies learning in a semi-bandit setting as well. This setting differs from the full information online problem as follows. In each round as we select the parameter $\rho_i$, we only observe losses for a single interval containing $\rho_i$ (as opposed to the entire domain). We call the set of these observable intervals the {\it feedback set}, and these provide a partition of the domain. The trade-off is slightly slower convergence, the regret bound for these approaches is weaker (it is $O(\sqrt{K})$ in the size $K$ of the feedback set instead of $O(\sqrt{\log K})$) but still converges to optimum as $\Tilde{O}(1/\sqrt{T})$.

\begin{algorithm}[t]
\caption{{Efficient Data-driven Graph-based SSL ($\lambda$)}}
\label{alg:ddsslsb}
\begin{algorithmic}[1]
\STATE {\bfseries Input:} Graphs $G_t$ with labeled and unlabeled nodes $(L_t,U_t)$ and node similarities $d(u,v)_{u,v\in L_t\cup U_t}$.
\STATE {\bfseries Hyperparameter:} step size parameter  $\lambda \in (0, 1]$.
\STATE {\bfseries Output:} Graph parameter $\rho_t$ for times $t=1,2,\dots,T$.
\STATE{Set $w_1(\rho)=1$ for all $\rho\in C$}
\FOR{$t=1,2,\dots,T$}
\STATE{$W_t:=\int_{C}w_t(\rho)d\rho$.}
\STATE{Sample $\rho$ with probability
        $p_{t}(\rho)=\frac{w_t(\rho)}{W_t}$, output as $\rho_t$.}
\STATE{Compute the feedback set $A^{(t)}(\rho)$ containing $\rho_t$. For example, for the min-cut objective use Algorithm \ref{algorithm: semi cut} to set $A^{(t)}(\rho)=\textsc{DynamicMinCut}(G_t,\rho_t,1/\sqrt{T})$, similarly for the harmonic objective use Algorithm \ref{algorithm: semi harmonic}.}
\STATE{Compute average loss function $l_t(\rho)=\frac{1}{|U_t|}\sum_{u\in U}l(A_{G_t(\rho),L_t,U_t}(u),\tau(u))$.}
\STATE{Set $\hat{l}_t(\rho)=\frac{\bI[\rho\in A^{(t)}(\rho)]}{\int_{A^{(t)}(\rho)}p_t(\rho)}l_t(\rho)$.}
\STATE{For each $\rho\in C, \text{ set }w_{t+1}(\rho)=e^{\lambda \hat{l}_t(\rho)}w_{t}(\rho)$.}
\ENDFOR
\end{algorithmic}
\end{algorithm}

For the case of learning the unweighted threshold graph, computing the feedback set containing a given $r$ is easy as we only need the next and previous thresholds from among the $O(n^2)$ values of pairwise distances where loss may be discontinuous in $r$. We present algorithms for computing the semi-bandit feedback sets (constant performance interval containing any $\sigma$) for the weighted graph setting (we will use Definition \ref{defn:g}c for concreteness, but the algorithms easily extend to Definition \ref{defn:g}b). We propose a hybrid combinatorial-continuous algorithm for the min cut objective and use continuous optimization for the harmonic objective (recall objectives from Table \ref{table:algos}). See Appendix \ref{app:sb} for background on the flow-cut terminology in Algorithm \ref{algorithm: semi cut}, and for a complete proof of its correctness.

\begin{algorithm}[t]
\caption{$\textsc{DynamicMinCut}(G,\sigma_0,\epsilon)$}
\label{algorithm: semi cut}
\begin{algorithmic}[1]
\STATE {\bfseries Input:} Graph $G$ with unlabeled nodes, query parameter $\sigma_0$, error tolerance $\epsilon$.
\STATE {\bfseries Output:} Piecewise constant interval containing $\sigma_0$.
\STATE{Use a max-flow algorithm to compute max-flow and min-cut $\C$ for $G(\sigma)$, $\sigma_h=\sigma_0$.}
\STATE{Compute the flow decomposition of the max-flow, $\F$.}
\STATE{Let $f_e$ be a unique {\it path flow} (i.e. along an $st$-path, Definition \ref{defn:pathflow}) through $e\in\C$.}
\STATE{Say $e$ is {\it augmentable} if flow $f_e$ can be increased by amount $w_e(\sigma)-w_e(\sigma_h)$ for some $\sigma>\sigma_h$. $e$ acts as the bottleneck for increasing the flow $f_e$.}
\STATE{Initialize $S$ to $\C$ (a set of saturated edges).}
\WHILE{All edges $e\in S$ are augmentable,}
\STATE{Increase flow in all $f_e$ for $e\in S$ to keep $e$ saturated.}
\STATE{Find first saturating edge $e_1\notin S$ for some $f_{e'}$ ($e'\in S$) and $\sigma'$ to within $\epsilon$.}
\STATE{Reassociate flow through $e_1,e'$ as $f_{e_1}$. $f_{e'}$ will now be along an alternate path in the residual capacities graph (Definition \ref{defn:residual}).}
\STATE{Add $e_1$ to $S$.}
\STATE{Set $\sigma_h=\sigma'$.}
\ENDWHILE
\STATE{Similarly find the start of the interval $\sigma_l$ by detecting saturation while reducing flows.}
\RETURN{$[\sigma_l,\sigma_h]$.}
\end{algorithmic}
\end{algorithm}

\begin{restatable}{theorem}{thmsb}\label{thm:sb}

For the mincut objective and exponential kernel (Definition \ref{defn:g}c), Algorithm \ref{algorithm: semi cut} outputs the interval containing $\sigma$ in time $O(n^3K(n,\epsilon))$, where $K(n,\epsilon)$ is the complexity of solving an exponential equation $\phi(y)=\sum_{i=1}^na_iy^{b_i}=0$ to within error $\epsilon$ (for example  $K(n,\epsilon)=O(n\log\log\frac{1}{\epsilon})$ using Newton's method with quadratic convergence).
\end{restatable}

\begin{proof}[Proof Sketch]
Let $L_1$ and $L_2$ denote the labeled points $L$ of different classes. To obtain the labels for $U$, we seek the smallest cut $(V_1, V\setminus V_1)$ of $G$ separating the nodes in $L_1$ and $L_2$. If $L_i\subseteq V_1$, label exactly the nodes in $V_1$ with label $i$. The loss function, $l(\sigma)$ gives the fraction of labels this procedure gets right for the unlabeled set $U$.

It is easy to see, if the min-cut is the same for two values of $\sigma$, then so is the loss function $l(\sigma)$. So we seek the smallest amount of change in $\sigma$ so that the mincut changes. Consider a fixed value of $\sigma=\sigma_0$ and the corresponding graph $G(\sigma_0)$. We can compute the max-flow on $G(\sigma)$, and simultaneously obtain a min-cut $(V_1,V \setminus V_1)$ in time $O(n^3)$. Clearly, all the edges in $\partial V_1$ are saturated by the flow. For each $e_i\in \partial V_1$, let $f_i$ denote the flow that saturated $e_i$. Note that the $f_i$ are distinct. Now as $\sigma$ is increased, we increment each $f_i$ by the additional capacity in the corresponding edge $e_i$, until an edge in $E\setminus \partial V_1$ saturates (at a faster rate than the flow through it). We now increment flows while keeping this edge saturated. The procedure stops when we can no longer find an alternate path for some flow among the unsaturated edges, which implies the existence of a new min-cut. This gives us a new critical value of $\sigma$.

Finally note that each time we perform step 10 of the algorithm, a new saturated edge stays saturated for all further $\sigma$ until the new cut is found. So we can do this at most $O(n^2)$ times. In each loop we need to obtain the saturation condition for $O(n)$ edges.
\end{proof}

For the harmonic objective, we can obtain similar efficiency (Algorithm \ref{algorithm: semi harmonic}). We seek points where $f_u(\sigma)=\frac{1}{2}$ for some $u\in U$ closest to given $\sigma_0$. For each $u$ we can find the local minima of $\left(f_u(\sigma)-\frac{1}{2}\right)^2$ or simply the root of $f_u(\sigma)-\frac{1}{2}$ using gradient descent or Newton's method. The gradient computation requires $O(n^3)$ time for matrix inversion, and we can obtain quadratic convergence rates for finding the root.

\begin{restatable}{theorem}{thmsb}\label{thm:sb}

For the harmonic function objective (Table \ref{table:algos}) and exponential kernel (Definition \ref{defn:g}c), Algorithm \ref{algorithm: semi harmonic} outputs the interval containing $\sigma$ within accuracy $\epsilon$ in time $O(n^4K_1(\epsilon))$, where $K_1(\epsilon)$ is the complexity of convergence for Newton's method ($K_1(\epsilon)=O(\log\log\frac{1}{\epsilon})$ under standard assumptions).
\end{restatable}

\begin{proof}
The key observation is that any boundary point $\sigma'$ of a piece (where the loss function is constant) has $f_u(\sigma')=\frac{1}{2}$ for some $u\in U$. This follows from continuity of $f_u(\sigma)$ (it is in fact differentiable). Algorithm \ref{algorithm: semi harmonic} simply estimates the locations of these $\sigma'$ closest to $\sigma_0$ for each $u$ (by using Newton's method) to find the root of $g_u(\sigma)=(f_u(\sigma)-\frac{1}{2})^2$. For each of $O(n)$ nodes, Algorithm \ref{algorithm: semi harmonic} computes the gradient and function value of $g_u(\sigma)$ in $O(n^3)$ time for different values of $\sigma$ until convergence, which gives the bound on time complexity.
\end{proof}

\begin{algorithm}[h]
\caption{$\textsc{HarmonicFeedbackSet}(G,\sigma_0,\epsilon)$}
\label{algorithm: semi harmonic}
\begin{algorithmic}[1]
\STATE {\bfseries Input:} Graph $G$ with unlabeled nodes, query parameter $\sigma_0$, error tolerance $\epsilon$.
\STATE {\bfseries Output:} Piecewise constant interval containing $\sigma_0$.
\STATE {Let $f_{U}=(D_{UU}-W_{UU})^{-1}W_{UL}f_L$ denote the harmonic objective minimizer, where $D_{ij}:=\bI[i=j]\sum_{k}W_{ik}$}.
\FORALL{$u\in U$}
\STATE{Let $g_u(\sigma)=(f_u(\sigma)-\frac{1}{2})^2$.}

\STATE{$\sigma_1=\sigma_0-\frac{g_u(\sigma_0)}{g'_u(\sigma_0)}$, where $g'_u(\sigma)$ is given by }
\begin{align*}
    \frac{\partial g_u}{\partial \sigma}&=2\left(f_u(\sigma)-\frac{1}{2}\right) \frac{\partial f_u}{\partial \sigma},\\
    \frac{\partial f}{\partial \sigma}&=(I-P_{UU}^{-1})\left(\frac{\partial P_{UU}}{\partial \sigma}f_U-\frac{\partial P_{UL}}{\partial \sigma}f_L\right),\\
    \frac{\partial P_{ij}}{\partial \sigma}&=\frac{\frac{\partial w(i,j)}{\partial \sigma}-P_{ij}\sum_{k\in L+U}\frac{\partial w(i,k)}{\partial \sigma}}{\sum_{k\in L+U}w(i,k)},\\
    \frac{\partial w(i,j)}{\partial \sigma}&=\frac{2w(i,j)d(i,j)^2}{\sigma^3},
\end{align*}
where $P=D^{-1}W$.

\STATE{$n=0$}

\WHILE{$|\sigma_{n+1}-\sigma_{n}|\ge \epsilon$}
\STATE{$n\leftarrow n+1$}
\STATE{$\sigma_{n+1}=\sigma_n-\frac{g_u(\sigma_n)}{g'_u(\sigma_n)}$}
\ENDWHILE
\STATE{$\sigma_u=\sigma_{n+1}$}
\ENDFOR
\STATE{$\sigma_l=\max_u\{\sigma_u\mid \sigma_u<\sigma_0\}$, $\sigma_h=\min_u\{\sigma_u\mid \sigma_u>\sigma_0\}$}
\RETURN{$[\sigma_l,\sigma_h]$.}
\end{algorithmic}
\end{algorithm}


\subsection{Distributional setting}\label{sec: distrib}
In the distributional setting, we are presented with instances of the problem assumed to be drawn from an unknown distribution $\D$ and want to learn the best value of the graph parameter $\rho$. We also assume we get all the labels for past instances ({\it full information}). A choice of $\rho$ uniquely determines the graph $G(\rho)$ and we use some algorithm $A_{G(\rho),L,U}$ to make predictions (e.g. minimizing the quadratic penalty score above) and suffers loss $l_{A(G(\rho),L,U)}:=\sum_Ul(A_{G(\rho),L,U}(u),\tau(u))$ which we seek to minimize relative to smallest possible loss by some graph in the hypothesis space, in expectation over the data distribution $\D$. 

We will show a divergence in the weighted and unweighted graph learning problems. {We analyze and provide asymptotically tight bounds for the pseudodimension of the set of loss functions (composed with the graph creation algorithm family and the optimization algorithm for predicting labels) prarmeterized by the graph family parameter $\rho$, i.e. $\mathcal{H}_{\rho}=\{l_{A(G(\rho),L,U)}\mid \rho\in\cP \}$.} For learning the unweighted threshold graphs, the pseudodimension is $O(\log n)$ which implies existence of an efficient algorithm with generalization guarantees in this setting. However, the pseudodimension is shown to be $\Omega(n)$ for the weighted graph setting, and therefore the smoothness assumptions are necessary for learning over the algorithm family. Both these bounds are shown to be tight up to constant factors (Section \ref{sec: pd}).

We establish uniform convergence guarantees in Section \ref{sec:uc}. For the unweighted graph setting, our pseudodimension bounds are sufficient for uniform convergence. We resort to bounding the Rademacher complexity in the weighted graph setting  which allows us to prove distribution dependent generalization guarantees, that hold under distributional niceness assumptions of Section \ref{sec:ol-d} (unlike pseudodimension which gives generalization guarantees that are worst-case over the
distribution). These guarantees are derived by extending our results in the dispersion setting, and follow under the same perturbation invariance assumptions.
Additional proofs and details from this section may be found in Appendix \ref{app:distrib}.

{The online learning results above only work for smoothed but adversarial instances, while the distributional learning sample complexity results based on pseudodimension work for any types (no smoothness needed) of independent and identically distributed instances. So these results are not superseded by the online learning results, the settings are strictly speaking incomparable, and the pseudodimension results in the distributional setting provide new upper and lower bounds for the problem.}

\subsubsection{Pseudodimension bounds}\label{sec: pd}

We can efficiently learn the unweighted graph with polynomially many samples. We show this by providing a bound on the pseudodimension of the set of loss functions $\mathcal{H}_r=\{l_{A(G(r),L,U)}\mid 0\le r<\infty \}$, where $G(r)$ is specified by Definition \ref{defn:g}a. Our bounds hold for both the min-cut and quadratic objectives (Table \ref{table:algos}).

\begin{theorem}\label{thm:pdub}
The pseudo-dimension of $\mathcal{H}_r$ is $O(\log n)$, where $n$ is the total number of (labeled and unlabeled) data points.
\end{theorem}

\begin{proof}
There are at most ${\binom{n}{2}}$ distinct distances between pairs of data points. As $r$ is increased from 0 to infinity, the graph changes only when $r$ corresponds to one of these distances, and so at most $\binom{n}{2}+1$ distinct graphs may be obtained.\\
Thus given set $\mathcal{S}$ of $m$ instances $(A^{(i)},L^{(i)})$, we can partition the real line into $O(mn^2)$ intervals such that all values of $r$ behave identically for all instances within any fixed interval. Since $A$ and therefore its loss is deterministic once $G$ is fixed, the loss function is a piecewise constant with only $O(n^2)$ pieces. Each piece can have a witness above or below it as $r$ is varied for the corresponding interval, and so the binary labeling of $\mathcal{S}$ is fixed in that interval. The pseudo-dimension $m$ satisfies $2^m\le O(mn^2)$ and is therefore $O(\log n)$.
\end{proof}

We can also show an asymptotically tight lower bound on the pseudodimension of $\mathcal{H}_r$. We do this by presenting a collection of graph thresholds and well-designed labeling instances which are shattered by the thresholds. An intuitive simplified sketch for the proof is included below, for full details see Appendix \ref{app:distrib}.

\begin{restatable}{theorem}{thmpdlb}\label{thm:pdlb}
The pseudo-dimension of $\mathcal{H}_r$ is $\Omega(\log n)$.
\end{restatable}

\begin{proof}[Proof Sketch]
{We have three labeled nodes, $a_1$ with label 0 and $b_1,b_2$ labeled 1, and $n'=O(n)$ unlabeled nodes $U=\{u_1,\dots,u_{n'}\}$. We can show that given a sequence $\{r_1,\dots,r_{n'}\}$ of values of $r$, it is possible to
construct an instance with suitable true labels of $U$ such that the loss as a function of $r$ oscillates above and below some witness as $r$ moves along the sequence of intervals $(r_i,r_{i+1})_{i\ge 0}$.}

\begin{figure}[h!]
    \centering
    
    \begin{tabular}{c }
    
    \begin{subfigure}[b]{0.22\textwidth}
    
\centering
         \includegraphics[width=\textwidth]{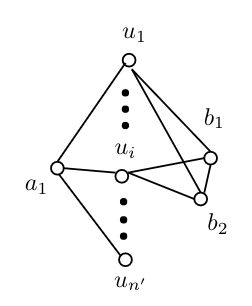}

 \subcaption{$G(r_i)$}
    \label{fig:lb-main1}
    \end{subfigure}
    \hspace{2.5cm}
    \begin{subfigure}[b]{0.33\textwidth}
    
\centering
         \includegraphics[width=0.8\textwidth]{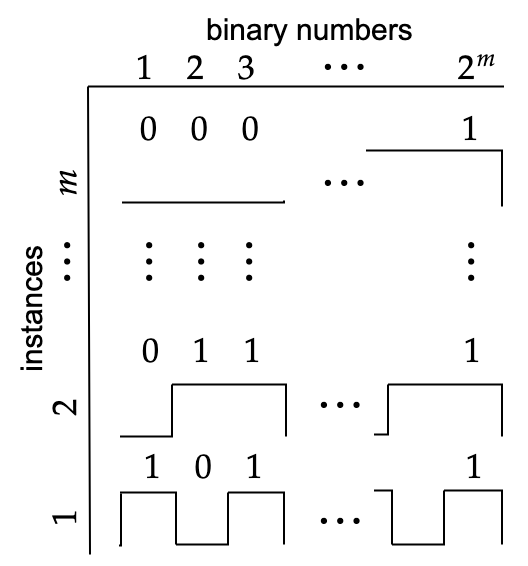}

 \subcaption{Loss oscillations in shattered instances}
    \label{fig:lb-main2}
    \end{subfigure}
    
    \end{tabular}

 \caption{Graphs $G(r)$ as $r$ is varied, for lower bound construction for pseudodimension of $\mathcal{H}_r$. Labels are set in the instances to match bit flips in the sequence of binary numbers as shown.}
    \label{fig:lb-main}
\end{figure}

{At the initial threshold $r_0$, all unlabeled points have a single incident edge, connecting to $a_1$, so all predicted labels are 0. As the threshold is increased to $r_i$, (the distances are set so that) $u_i$ gets connected to both nodes with label 1 and therefore its predicted label changes to 1. If the sequence of nodes $u_i$ is alternately labeled, the loss decreases and increases alternately as all the predicted labels turn to 1 as $r$ is increased to $r_{n'}$. This oscillation between a high and a low value can be achieved for any subsequence of distances $r_1,\dots,r_{n'}$, and a witness may be set as a loss value between the oscillation limits. By precisely choosing the subsequences so that the oscillations align with the bit flips in the binary digit sequence, we can construct $m$ instances which satisfy the $2^m$ shattering constraints.}
\end{proof}

For learning weighted graphs, we can show a $\Theta(n)$ bound on the pseudodimension of the set of loss functions. We show this for the loss functions for learning graphs with exponential kernels, $\mathcal{H}_{\sigma}=\{l_{A(G(\sigma),L,U)}\mid 0\le \sigma<\infty \}$, where $G(\sigma)$ is specified by Definition \ref{defn:g}c. For simplicity of presentation, the lower bound will be shown for the min-cut objective.

\begin{theorem}\label{thm:pdubs}
The pseudo-dimension of $\mathcal{H}_{\sigma}$ is $O(n)$, where $n$ is the total number of (labeled and unlabeled) data points.
\end{theorem}
\begin{proof}
This bound trivially follows by noting that we have $n$ vertices and therefore only $2^n$ possible labelings in an instance. Thus, for $m$ problems, $2^m\le m2^n$ gives $m=O(n)$.
\end{proof}

\begin{restatable}{theorem}{thmpdlbs}\label{thm:pdlbs}
The pseudo-dimension of $\mathcal{H}_{\sigma}$ is $\Omega(n)$.
\end{restatable}
\begin{proof}[Proof Sketch]
We first show that we can carefully set the distances between the examples so that we have at least $2^{\Omega(n)}$ intervals of $\sigma$ such that $G(\sigma)$ has distinct vertex sets as min-cuts in each interval. We start with a pair of labeled nodes of each class, and a pair of unlabeled nodes which may be assigned either label depending on $\sigma$.  We then build the graph in $N=(n-4)/2$ stages, adding two new nodes at each step with meticulously chosen distances from existing nodes. Before adding the $i$th pair $x_i,y_i$ of nodes, there will be 
$2^{i-1}$ intervals of $\sigma$ such that the intervals correspond to distinct min-cuts which result in all possible labelings of $\{x_1,\dots,x_{i-1}\}$. Moreover, $y_j$ will be labeled differently from $x_j$ in each of these intervals. The edges to the new nodes will ensure that the cuts that differ exactly in $x_i$ will divide each of these intervals giving us $2^i$ intervals where distinct mincuts give all labelings of $\{x_1,\dots,x_{i}\}$, and allowing an inductive proof. The challenge is that we only get to set $O(i)$ edges in round $i$ but seek properties about $2^{i}$ cuts, so we must do this carefully. The full construction and proof is presented in Appendix \ref{app:distrib}. 

Finally, we use this construction to generate a set of $\Theta(n)$ instances that correspond to cost functions that oscillate in a manner that helps us pick $2^{\Omega(n)}$ values of $\sigma$ that shatter the samples.
\end{proof}


\subsubsection{Uniform convergence}\label{sec:uc}
Our results above implies a uniform convergence guarantee for the offline distributional setting, for both weighted and unweighted graph families. For the unweighted case, we can use the pseudodimension bounds from section \ref{sec: pd}, and for the weighted case we use dispersion guarantees from section \ref{sec: ol}. For either case it suffices to bound the empirical Rademacher complexity. We will need the following theorem (slightly rephrased) from \cite{balcan2018dispersion}.

\begin{theorem}
Let $\F=\{f_\rho:\X\rightarrow [0,1], \rho\in\C\subset\R^d\}$ be a parametereized family of functions, where $\C$ lies in a
ball of radius $R$. For any set $\s=\{x_i,\dots,x_T\}\subseteq\X$, suppose the functions $u_{x_i}(\rho) = f_\rho(x_i)$ for $i \in [T]$
are piecewise L-Lipschitz and $\beta$-dispersed. Then $\hat{R}(\F,\s)\le O(\min\{\sqrt{(d/T)\log RT}+LT^{-\beta},\sqrt{Pdim(\F)/T}\})$.
\end{theorem}

Now, using classic results from learning theory, we can conclude that Empirical Risk Minimization has good generalization under the distribution. 

\begin{theorem}
For both weighted and unweighted graph $w(u,v)$ defined above, with probability at least $1-\delta$, the average loss on any sample $x_1,\dots,x_T\sim D^T$, the loss suffered w.r.t. to any parameter $\rho\in\R^d$ satisfies $|\frac{1}{T}\sum_{i=1}^Tl_\rho(x_i)-\E_{x\sim D}l_{\rho}(x)|\le O\left(\sqrt{\frac{d\log T\log 1/\delta}{T}}\right)$.
\end{theorem}

\section{Extension to active learning}


So far we have considered an oracle which gives labeled and unlabeled points, where the labels of the unlabeled points possibly revealed later but it does not allow us to select a subset of the unlabeled points for which we obtain the labels. A natural question to ask is if we can further reduce the need for labels by active learning. The S$^2$ algorithm of \cite{dasarathy2015s2} allows efficient active learning given the graph, but we would like to learn the graph itself. We will present an algorithm which extends the approach of \cite{zhu2003combining} to do data-driven active learning with a constant budget $L$ of labels in a general (non-realizable) setting. 


We have the same problem setup as before (Section \ref{sec:notation}), except we will now learn the set of labeled examples $L$ via an active learning procedure $A'$. We set $\ell:=|L|$ as the {\it budget} for the active learning which we will think of as a small constant. $A'$ takes as input the graph $G$, the budget $\ell$ and outputs a set of nodes $L$ with $|L|=\ell$ to query labels for. In our budgeted active learning setup we provide all the queries in a single batch, we consider sequential querying in the next section. Let $A'(G,\ell)$ denote the set $L$ of labeled examples obtained using $A'$. In addition, we assume we are given a set of initially labeled points $L_0$ (not included in the budget) consisting of some example from each class of labels. This is needed for known graph-based active learning procedures to work \cite{zhu2003combining}, our results in the next section imply upper bounds on the size of $L_0$ for a purely unsupervised initialization. As with the semi-supervised learning above, the effectiveness of the active learning also heavily depends on how well the graph captures the label similarity.

We are interested in learning the graph parameter $\rho$ which minimizes the loss $l_{A(G(\rho),L_0\cup A'(G(\rho),\ell),U)}$ where $U$ is the set of nodes not selected for querying by $A'$. That is, we seek the graph that works best with the combined active semi-supervised procedure for predicting the labels (by learning over multiple problem instances). We will consider the harmonic objective minimization as the semi-supervised algorithm $A$. Several reasonable heuristics for $A'$ may be considered, we will restrict our attention to Algorithm \ref{alg: budgeted al} which adapts the greedy algorithm of \cite{zhu2003combining} to our budgeted setting. Our main contribution in this section is to prove dispersion for the sequence of loss functions, which implies learnability of the graph for the composite active semi-supervised procedure. The loss function is piecewise constant in the parameter $\rho$, with discontinuities corresponding to changes in algorithmic decisions in selecting the labeling set $L$ and subsequent predictions for the unlabeled examples. The following theorem establishes this for polynomial kernels, extensions to other kernels may be made as before.

\begin{restatable}{theorem}{thmpolyal}\label{thm:dispersion-bal}
Let $l_1,\dots, l_T:\R\rightarrow\R$ denote an independent sequence of losses as a function of parameter $\Tilde{\alpha}$, when the graph is created using a polynomial kernel $w(u,v)=(\Tilde{d}(u,v)+\Tilde{\alpha})^d$ and labeled by Algorithm \ref{alg: budgeted al} followed by predicting the remaining labels by optimizing the quadratic objective $\sum_{u,v}w(u,v)(f(u)-f(v))^2$. If $\Tilde{d}(u,v)$ follows a $\kappa$-bounded distribution with a closed and bounded support, the sequence is $\frac{1}{2}$-dispersed.
\end{restatable}

\begin{proof}[Proof Sketch] The proof reuses ideas from the proof of Theorem \ref{thm:dispersion} with additional insights to handle the dependence of the loss functions on the budgeted active learning procedure. Discontinuities in the loss function may only occur when the label set from the active procedure changes (which only happens if $u^S=u^T$ for some candidate subsets $S\ne T$ in Algorithm \ref{alg: budgeted al}) or when the semi-supervised prediction changes ($f(u)=1/2$ for the soft labeling using the full labeled set after active learning). We show that both kinds of discontinuities lie along roots of polynomials with bounded discontinuities and are therefore dispersed. Detailed argument appears in Appendix \ref{app:bal}
\end{proof}

\begin{algorithm}[t!]
\caption{$\textsc{BudgetedActiveLearning}(G,\ell)$}
\label{alg: budgeted al}
\begin{algorithmic}[1]
\STATE {\bfseries Input:} Graph $G$ with unlabeled nodes, initial labels $L_0$, label budget $\ell$.
\STATE {\bfseries Output:} Nodes for label query $L$, $|L|=\ell$.
\STATE{Compute soft labeling $f$ to minimize the harmonic objective $f^T(D-W)f$ using initial labels $L_0$.}
\FOR{each subset $S$ of unlabeled nodes, $|S|=\ell$}
\FOR{each labeling $L_S\in\{0,1\}^{\ell}$ of $S$}
\STATE Estimate probability of $L_S$, $p^{L_S}=\Pi_{s\in S}[L_S(s)f_s+(1-L_S(s))(1-f_s)]$.
\STATE Compute soft labels $f^{L_S}$, by minimizing the harmonic objective with labels $L_S$ for $S$ (in addition to $L_0$).
\STATE Estimate uncertainty as $u^{L_S}=\sum_i\frac{1}{2}-\lvert \frac{1}{2}-f^{L_S}_i\rvert$.

\ENDFOR
\ENDFOR
\STATE Return $\argmin_{S}u^{S}$ where $u^S:=\sum_{L_S}p^{L_S}u^{L_S}$.
\end{algorithmic}
\end{algorithm}

We remark that while we have only considered learning the graph family with the active learning procedure of \cite{zhu2003combining}, but it would be interesting to extend our results to other similar active learning algorithms \cite{long2008graph,zhao2008scalable}.

\section{Experiments}
In this section we evaluate the performance of our learning procedures when finding application-specific semi-supervised learning algorithms (i.e. graph parameters). Our experiments demonstrate that the best parameter for different applications
varies greatly, and that the techniques presented in this paper can lead to large gains. We will look at image classification based on standard pixel embedding for different datasets. 
\begin{figure*}[!t]
    \centering
    \begin{subfigure}[b]{0.32\textwidth}
    \centering
         \includegraphics[width=\textwidth]{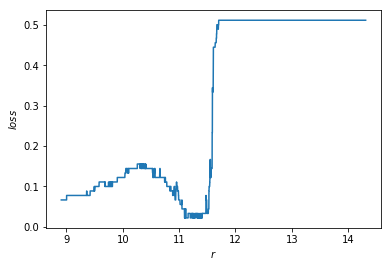}
    \caption{MNIST}
  \end{subfigure}
    \begin{subfigure}[b]{0.32\textwidth}
    \centering
         \includegraphics[width=\textwidth]{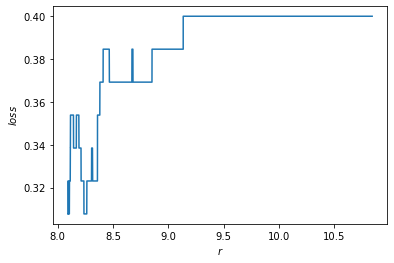}
    \caption{Omniglot}
  \end{subfigure}
    \begin{subfigure}[b]{0.32\textwidth}
    \centering
         \includegraphics[width=\textwidth]{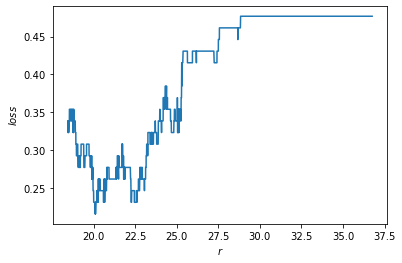}
    \caption{CIFAR-10}
  \end{subfigure}
  \caption{Loss for different unweighted graphs as a function of the threshold $r$.}
\label{fig:r}
\end{figure*}

\begin{figure*}[!t]
    \centering
    \begin{subfigure}[b]{0.32\textwidth}
    \centering
         \includegraphics[width=\textwidth]{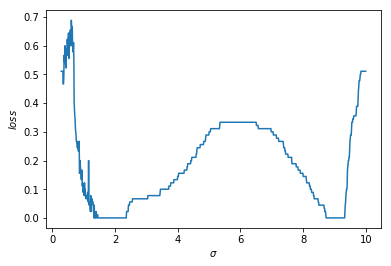}
    \caption{MNIST}
  \end{subfigure}
    \begin{subfigure}[b]{0.32\textwidth}
    \centering
         \includegraphics[width=\textwidth]{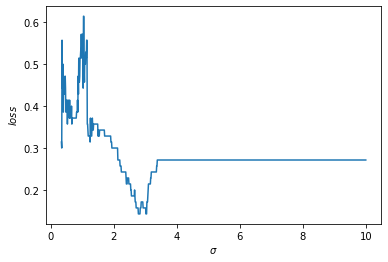}
    \caption{Omniglot}
  \end{subfigure}
    \begin{subfigure}[b]{0.32\textwidth}
    \centering
         \includegraphics[width=\textwidth]{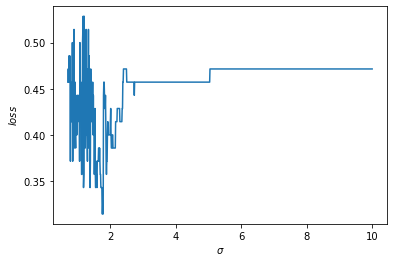}
    \caption{CIFAR-10}
  \end{subfigure}
  \caption{Loss for different weighted graphs as a function of the parameter $\sigma$.}
\label{fig:s}
\end{figure*}



{\it Setup}: We consider the task of semi-supervised classfication on image datasets. We restrict our attention to binary classification 
and pick two classes for each data set. We then draw random subsets of the dataset (with class restriction) of size $n=100$ and randomly select $L$ ($10\le L\le 20$) examples for labeling. For any data subset $S$, we measure distance between any pairs of images
using the $L_2$ distance between their pixel intensities. We would like to determine data-specific parameters $r$ and $\sigma$ which lead to good weighted and unweighted graphs for semi-supervised learning on the datasets. We will optimize the harmonic function objective (Table \ref{table:algos}) and round the fractional labels $f$ to make our predictions.

{\it Data sets}: We use three popular benchmark datasets --- MNIST, Omniglot and CIFAR-10. The MNIST dataset \citep{lecun1998gradient} contains images of hand-written digits from 0 to 9 as $28 \times 28$ binary images, with 5000 training examples for each class. We consider examples with labels 0 or 1.
We generate a random semi-supervised learning instance from this data by sampling $100$ random examples and further sampling $L=10$ random examples from the subset for labeling. Omniglot \citep{lake2015human} has
$105 \times 105$ binary images of handwritten characters across 30 alphabets with 19,280 examples. We consider the task of distinguishing alphabets 0 and 1, and set $L=20$ in this setting. CIFAR-10 \citep{szegedy2015going} has $32\times 32$ color images (an integer value in $[0,255]$ for each of three colors) for object recognition among 10 classes. Again we consider objects 0 and 1 and set $L=20$.

{\it Results and discussion}: For the MNIST dataset we get optimal parameters with near-perfect classification even with small values of $L$, while the error of the optimal parameter is $\sim 0.2-0.3$ even with larger values of $L$, indicating  differences in the inherent difficulties of the classification tasks (like label noise and how well separated the classes are). We examine the full variation of performance of graph-based semi-supervised learning for all possible graphs $G(r)$ ($r_{\min}<r<r_{\max}$) and $G(\sigma)$ for $\sigma\in[0,10]$ (Figures \ref{fig:r}, \ref{fig:s}). The losses are piecewise constant and can have large discontinuities in some cases. The optimal parameter values vary with the dataset, but we observe at least 10\% gap in performance between optimal and suboptimal values within the same dataset.

Another interesting observation is the variation of optima across subsets, indicating transductively optimal parameters may not generalize well. We plot the variation of loss with graph parameter $\sigma$ for several subsets of the same size $N=100$ for MNIST and Omniglot datasets in Figure \ref{fig:subsets}. In MNIST we have two optimal ranges in most subsets but only one shared optimum (around $\sigma=2$) across different subsets. This indicates that local search based techniques that estimate the optimal parameter values on a given data instance may lead to very poor performance on unseen instances. The CIFAR-10 example further shows that the optimal algorithm may not be easy to empirically discern.

\begin{figure*}[!t]
    \centering
    \begin{subfigure}[b]{0.32\textwidth}
    \centering
         \includegraphics[width=\textwidth]{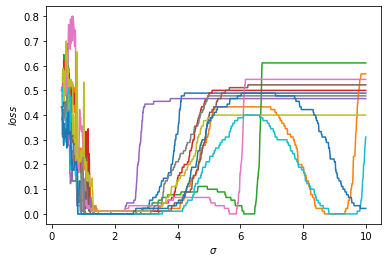}
    \caption{MNIST}
  \end{subfigure}
    \begin{subfigure}[b]{0.32\textwidth}
    \centering
         \includegraphics[width=\textwidth]{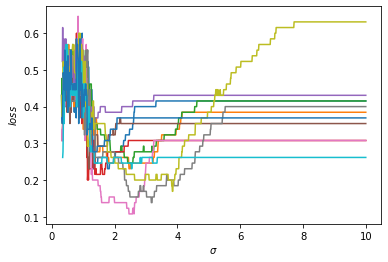}
    \caption{Omniglot}
  \end{subfigure}
    \begin{subfigure}[b]{0.32\textwidth}
    \centering
         \includegraphics[width=\textwidth]{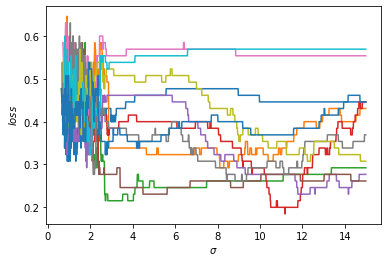}
    \caption{CIFAR-10}
  \end{subfigure}
  \caption{Comparing different subsets of the same problem.}
\label{fig:subsets}
\end{figure*}

\begin{figure*}[!t]
    \centering
    \begin{subfigure}[b]{0.32\textwidth}
    \centering
         \includegraphics[width=\textwidth]{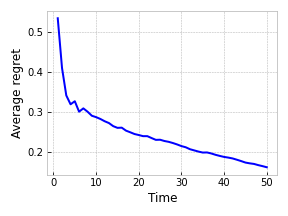}
    \caption{MNIST}
  \end{subfigure}
    \begin{subfigure}[b]{0.32\textwidth}
    \centering
         \includegraphics[width=\textwidth]{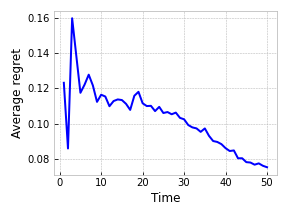}
    \caption{Omniglot}
  \end{subfigure}
    \begin{subfigure}[b]{0.32\textwidth}
    \centering
         \includegraphics[width=\textwidth]{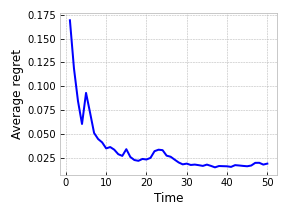}
    \caption{CIFAR-10}
  \end{subfigure}
  \caption{Average regret vs. $T$ for online learning of parameter $\sigma$}
\label{fig:ol}
\end{figure*}


We also implement our online algorithms and compute the average regret (i.e. excess error in predicting labels of unlabeled examples over the best parameter in hindsight) for finding the optimal graph parameter $\sigma$ for the different datasets. To obtain smooth curves we plot the average over 50 iterations for learning from 50 problem instances each ($T=50$, Figure \ref{fig:ol}). We observe fast convergence to the optimal parameter regret for all the datasets considered. The starting part of these curves ($T=0$) indicates regret for randomly setting the graph parameters, averaged over iterations, which is strongly outperformed by our learning algorithms as they learn from problem instances.






\section{Acknowledgments}
This material is based on work supported by the National Science Foundation under grants CCF-1535967, CCF-1910321, IIS-1618714, IIS-1901403, and SES-1919453; the Defense Advanced Research Projects Agency under cooperative agreement HR00112020003; an AWS Machine Learning Research Award; an Amazon Research Award; a Bloomberg Research Grant; a Microsoft Research Faculty Fellowship. The views expressed in this work do not
necessarily reflect the position or the policy of the Government and no official endorsement should be inferred.


\bibliography{mybib}{}
\bibliographystyle{plainnat}

\newpage
\section*{Appendix}
\appendix
\section{Dispersion and Online learning}\label{app:ol}

In this appendix we include details of proofs and algorithms from section \ref{sec: ol}.




\subsection{A general tool for analyzing dispersion}\label{app:disp-recipe}

If the weights of the graph are given by a polynomial kernel $w(u,v)=(\Tilde{d}(u, v)+\Tilde{\alpha})^d$, we can apply the general tool developed by \citet{dick2020semi} to learn $\Tilde{\alpha}$, which we summarize below.

\begin{enumerate}
    \item  Bound the probability density of the random set of
discontinuities of the loss functions.
\item  Use a VC-dimension based uniform convergence argument to transform this into a bound on the
dispersion of the loss functions.
\end{enumerate}

Formally, we have the following theorems from \cite{dick2020semi}, which show how to use this technique when the discontinuities are roots of a random polynomial.

\begin{theorem}[\cite{dick2020semi}]\label{thm:poly-roots}
Consider a random degree $d$ polynomial
$\phi$ with leading coefficient 1 and subsequent coefficients
which are real of absolute value at most $R$, whose joint
density is at most $\kappa$. There is an absolute constant $K$
depending only on $d$ and $R$ such that every interval $I$ of
length $\le\epsilon$ satisfies Pr($\phi$ has a root in $I$) $\le \kappa\epsilon/K$.
\end{theorem}

\begin{theorem}[\cite{dick2020semi}]\label{thm:VC-bound}
 Let $l_1, \dots, l_T : \R \rightarrow \R$ be independent piecewise $L$-Lipschitz functions, each having at most $K$ discontinuities. Let $D(T, \epsilon, \rho) = |\{1 \le t \le T \mid l_t\text{ is not }L\text{-Lipschitz on }[\rho - \epsilon, \rho + \epsilon]\}|$
be the number of functions that are not
$L$-Lipschitz on the ball $[\rho - \epsilon, \rho + \epsilon]$. Then we
have $E[\max_{\rho\in\R} D(T, \epsilon, \rho)] \le \max_{\rho\in\R} E[D(T, \epsilon, \rho)] +
O(\sqrt{T \log(TK)})$.
\end{theorem}

We will now use Theorems \ref{thm:poly-roots} and \ref{thm:VC-bound} to establish dispersion in our setting. We first need a simple lemma about $\kappa$-bounded distributions. We remark that similar properties have been proved in \cite{balcan2018dispersion,dick2020semi}, in other problem contexts. {Specifically, \cite{balcan2018dispersion} show the lemma for a ratio of random variables, $Z=X/Y$, and \cite{dick2020semi} establish it for the sum $Z=X+Y$ but for independent variables $X,Y$.}

\begin{lemma}\label{lem:bounded}
Suppose $X$ and $Y$ are real-valued random variables taking values in $[m, m + M]$ and
$[m', m'+M']$ for some $m,m',M,M'\in \R^+$ and suppose that their joint distribution is $\kappa$-bounded. Then,
\begin{enumerate}
    \item[(i)] $Z=X+Y$ is drawn from a $K_1\kappa$-bounded distribution, where $K_1\le \min\{M,M'\}$.
    \item[(ii)] $Z=XY$ is drawn from a $K_2\kappa$-bounded distribution, where $K_2\le \min\{M/m,M'/m'\}$.
\end{enumerate}
\end{lemma}
\begin{proof} Let $f_{X,Y}(x,y)$ denote the joint density of $X,Y$. 
\begin{enumerate}
    \item[(i)] The case where $X,Y$ are independent has been studied (Lemma 25 in \cite{dick2020semi}), the following is slightly more involved.
    The cumulative density function for $Z$ is given by
    \begin{align*}
        F_Z(z)&=\bP(Z\le z)=\bP(X+Y\le z)=\bP(X\le z-Y)\\
        &=\int_{m'}^{m'+M'}\int_{m}^{z-y}f_{X,Y}(x,y)dxdy.
    \end{align*}
    The density function for $Z$ can be obtained using Leibniz's rule as
    \begin{align*}
        f_Z(z)=\frac{d}{dz}F_Z(z)&=\frac{d}{dz}\int_{m'}^{m'+M'}\int_{m}^{z-y}f_{X,Y}(x,y)dxdy\\
        &=\int_{m'}^{m'+M'}\left(\frac{d}{dz}\int_{m}^{z-y}f_{X,Y}(x,y)dx\right)dy\\
        &=\int_{m'}^{m'+M'}f_{X,Y}(z-y,y)dy\\
        &\le M'\kappa.
    \end{align*}
    A symmetric argument shows that $f_Z(z)\le M\kappa$, together with above this completes the proof.
    
    \item[(ii)] The cumulative density function for $Z$ is given by
    \begin{align*}
        F_Z(z)&=\bP(Z\le z)=\bP(XY\le z)=\bP(X\le z/Y)\\
        &=\int_{m'}^{m'+M'}\int_{m}^{z/y}f_{X,Y}(x,y)dxdy.
    \end{align*}
    The density function for $Z$ can be obtained using Leibniz's rule as
    \begin{align*}
        f_Z(z)=\frac{d}{dz}F_Z(z)&=\frac{d}{dz}\int_{m'}^{m'+M'}\int_{m}^{z/y}f_{X,Y}(x,y)dxdy\\
        &=\int_{m'}^{m'+M'}\left(\frac{d}{dz}\int_{m}^{z/y}f_{X,Y}(x,y)dx\right)dy\\
        &=\int_{m'}^{m'+M'}\frac{1}{y}f_{X,Y}(z/y,y)dy\\
        &\le \int_{m'}^{m'+M'}\frac{1}{m'}f_{X,Y}(z/y,y)dy\\
        &\le \frac{M'}{m'}\kappa.
    \end{align*}
    Similarly we can show that $f_Z(z)\le M\kappa/m$, together with above this completes the proof.
\end{enumerate}
\end{proof}

\thmpoly*

\begin{proof} $w(u,v)$ is a polynomial in $\Tilde{\alpha}$ of degree $d$ with coefficient of $\Tilde{\alpha}^i$ given by $c_i=D_{d,i}\Tilde{d}(u,v)^{E_{d,i}}$ for $i\in[d]$. Since the support of $\Tilde{d}(u,v)$ is closed and bounded, we have $m \le \Tilde{d}(u,v)\le M$ with probability 1 for some $M>1,m>0$ (since $\Tilde{d}(u,v)$ is a metric, $\Tilde{d}(u,v)>0$ for $u\ne v$).

To apply Theorem \ref{thm:poly-roots}, we note that we have an upper bound on the coefficients, $R<(dM)^d$. Moreover, if $f(x)$ denotes the probability density of $d(u,v)$ and $F(x)$ its cumulative density,
\begin{align*}
    \bP(c_i\le x_i)=\bP\left(D_{d,i}\Tilde{d}(u,v)^{E_{d,i}}\le x_i\right) = \bP\left(\Tilde{d}(u,v)\le \left(\frac{x_i}{D_{d,i}}\right)^{1/E_{d,i}}\right)=F\left(\left(\frac{x_i}{D_{d,i}}\right)^{1/E_{d,i}}\right).
\end{align*}
Thus, 
\[\bP(c_i\le x_i\text{ for each }i\in[d])=F\left(\min_i\left(\frac{x_i}{D_{d,i}}\right)^{1/E_{d,i}}\right).\]

The joint density of the coefficients is therefore $K\kappa$-bounded where $K$ only depends on $d,m$. ($K\le \max_i {D_{d,i}}^{-1/E_{d,i}}m^{-1+1/E_{d,i}}$).

Consider the harmonic solution of the quadratic objective \cite{zhu2003semi} which is given by $f_{U}=(D_{UU}-W_{UU})^{-1}W_{UL}f_L$. For any $u\in U$, $f(u)=1/2$ is a polynomial equation in $\Tilde{\alpha}$ with degree at most $nd$. The coefficients of these polynomials are formed by multiplying sets of weights $w(u,v)$ of size up to $n$ and adding the products, and are also bounded density on a bounded support (using above observation in conjunction with Lemma \ref{lem:bounded}). The dispersion result now follows by an application of Theorems \ref{thm:poly-roots} and \ref{thm:VC-bound}. The regret bound is implied by results from \cite{balcan2018dispersion,sharma2020learning}.
\end{proof}

\subsection{Dispersion for roots of exponential polynomials}\label{app:exp}
In this section we will extend the applicability of the dispersion analysis technique from Appendix \ref{app:disp-recipe} to exponential polynomials, i.e. functions of the form $\phi(x)=\sum_{i=1}^na_ie^{b_ix}$. We will now extend the analysis to obtain similar results when using the exponential kernel $w(u,v)=e^{-||u-v||^2/\sigma^2}$. The results of \citet{dick2020semi} no longer directly apply as the points of discontinuity are no longer roots of polynomials. To this end, we extend and generalize arguments from \cite{dick2020semi} below. We need to generalize Theorem \ref{thm:poly-roots} to exponential polynomials below.

\begin{theorem}\label{thm:exproots}
Let $\phi(x)=\sum_{i=1}^na_ie^{b_ix}$ be a random function, such that coefficients $a_i$ are real and of magnitude at most $R$, and distributed with joint density at most $\kappa$. Then for any interval $I$ of width at most $\epsilon$, P($\phi$ has a zero in $I$)$\le \Tilde{O}(\epsilon)$ (dependence on $b_i,n,\kappa,R$ suppressed).
\end{theorem}

\begin{proof}
For $n=1$ there are no roots, so assume $n>1$. Suppose $\rho$ is a root of $\phi(x)$. Then $\mathbf{a}=(a_1,\dots,a_n)$ is orthogonal to $\varrho(\rho)=(e^{b_1\rho},\dots,e^{b_n\rho})$ in $\R^n$. For a fixed $\rho$, the set $S_\rho$ of coefficients $\mathbf{a}$ for which $\rho$ is a root of $\phi(y)$ lie along an $n-1$ dimensional linear subspace of $\R^n$. Now $\phi$ has a root in any interval $I$ of length $\epsilon$, exactly when the coefficients lie on $S_\rho$ for some $\rho\in I$.  The desired probability is therefore upper bounded by $\max_{\rho}\textsc{Vol}(\cup S_y\mid y\in [\rho - \epsilon, \rho + \epsilon])/\textsc{Vol}(S_y\mid y\in \R)$ which we will show to be $\Tilde{O}(\epsilon)$. The key idea is that if $|\rho-\rho'|<\epsilon$, then $\varrho(\rho)$ and $\varrho(\rho')$ are within a small angle $\theta_{\rho,\rho'}=\Tilde{O}(\epsilon)$ for small $\epsilon$ (the probability bound is vacuous for large $\epsilon$). But any point in $S_{\rho}$ is at most $\Tilde{O}(\theta_{\rho,\rho'})$ from a point in $S_{\rho'}$, which implies the desired bound (similar arguments to Theorem \ref{thm:poly-roots}).

We will now flesh out the above sketch. Indeed,
\begin{align*}
    \sin\theta_{\rho,\rho'}=
    \sqrt{1-
    \frac{\left(\langle\varrho(\rho),\varrho(\rho')\rangle\right)^2}
    {\norm{\varrho(\rho)}\norm{\varrho(\rho')}}}
    =\sqrt{1-\frac{\left(\sum_i e^{b_i\rho}e^{b_i\rho'}\right)^2}{\sum_ie^{2b_i\rho}\sum_ie^{2b_i\rho'}}}
    =\sqrt{\frac{\sum_{i\ne j} e^{2(b_i\rho+b_j\rho')}-e^{(b_i+b_j)(\rho+\rho')}}{\sum_ie^{2b_i\rho}\sum_ie^{2b_i\rho'}}}.
\end{align*}

Now, for $\rho'=\rho+\varepsilon$, $|\varepsilon|<\epsilon$,

\begin{align*}
    \sin\theta_{\rho,\rho'}
    =\sqrt{\frac{\sum_{i\ne j} e^{2(b_i\rho+b_j\rho+b_j\varepsilon)}-e^{(b_i+b_j)(2\rho+\varepsilon)}}{\sum_ie^{2b_i\rho}\sum_ie^{2b_i\rho'}}}
    =\sqrt{\frac{\sum_{i\ne j} e^{2\rho(b_i+b_j)}(e^{2b_j\varepsilon}-e^{(b_i+b_j)\varepsilon})}{\sum_ie^{2b_i\rho}\sum_ie^{2b_i\rho'}}}.
\end{align*}
Using the Taylor's series approximation for $e^{2b_j\varepsilon}$ and $e^{(b_i+b_j)\varepsilon}$, we note that the largest terms that survive are quadratic in $\varepsilon$. $\sin\theta_{\rho,\rho'}$, and therefore also $\theta_{\rho,\rho'}$, is $\Tilde{O}(\epsilon)$.

Next it is easy to show that any point in $S_{\rho}$ is at most $\Tilde{O}(\theta_{\rho,\rho'})$ from a point in $S_{\rho'}$. For $n=2$, $S_{\rho}$ and $S_{\rho'}$ are along lines orthogonal to $\rho$ and $\rho'$ and are thus themselves at an angle $\theta_{\rho,\rho'}$. Since we further assume that the coefficients are bounded by $R$, any point on $S_{\rho}$ is within $O(R\theta_{\rho,\rho'})=\Tilde{O}(\theta_{\rho,\rho'})$ of the nearest point in $S_{\rho'}$. For $n>2$, consider the 3-space spanned by $\rho$, $\rho'$ and an arbitrary $\varsigma\in S_{\rho}$. $S_{\rho}$ and $S_{\rho'}$ are along 2-planes in this space with normal vectors $\rho,\rho'$ respectively. Again it is straightforward to see that the nearest point in the projection of $S_{\rho'}$ to $\varsigma$ is $\Tilde{O}(\theta_{\rho,\rho'})$.

The remaining proof is identical to that of Theorem \ref{thm:poly-roots} (see Theorem 18 of \cite{dick2020semi}), and is omitted for brevity.


\end{proof}

We will also need the following lemma for the second step noted above, i.e. obtain a result similar to Theorem \ref{thm:VC-bound} for exponential polynomials.


\begin{lemma}\label{lem:exproot}
The equation $\sum_{i=1}^na_ie^{b_ix} = 0$ where $a_i,b_i\in \R$ has at most $n-1$ distinct solutions $x\in\R$.
\end{lemma}

\begin{proof}
We will use induction on $n$. It is easy to verify that there is no solution for $n=1$. We assume the statement holds for all $1\le n\le N$.
Consider the equation $\phi_{N+1}(x)=\sum_{i=1}^{N+1}a_ie^{b_ix}=0$. WLOG $a_1\ne 0$ and we can write
\[\phi_{N+1}(x)=\sum_{i=1}^{N+1}a_ie^{b_ix}=a_1e^{b_1x}\left(1+\sum_{i=2}^{N+1}\frac{a_i}{a_1}e^{(b_i-b_1)x}\right)=a_1e^{b_1x}\left(1+g(x)\right).\]
By our induction hypothesis, $g'(0)=0$ has at most $N-1$ solutions, and so $(1+g(x))'$ has at most $N-1$ roots. By Rolle's theorem, $(1+g(x))$ has at most $N$ roots, and therefore $\phi_{N+1}(x)=0$ has at most $N$ solutions.
\end{proof}

Lemma \ref{lem:exproot} implies that Theorem \ref{thm:VC-bound} may be applied. The number of discontinuities may be exponentially high in this case. Indeed solving the quadratic objective can result in an exponential equation of the form in Lemma \ref{lem:exproot} with $O(|U|^n)$ terms.

\subsection{Learning several metrics simultaneously}\label{app:multimetric}
We start by getting a couple useful definitions out of the way.
\begin{defn}[Homogeneous algebraic hypersurface] An algebraic hypersurface is an algebraic variety (a system of polynomial equations) that may be defined by a single implicit equation of the form $p(x_1,\dots,x_n)=0$, where $p$ is a multivariate polynomial. The degree $d$ of the algebraic hypersurface is the total degree of the polynomial $p$. We say that the algebraic hypersurface is homogeneous if $p$ is a homogeneous polynomial, i.e. $p(\lambda x_1,\dots,\lambda x_m)=\lambda^dp(x_1,\dots,x_n)$.
\end{defn}
In the following we will refer to homogeneous algebraic hypersurfaces as simply algebraic hypersurfaces. We will also need the standard definition of set shattering, which we restate in our context as follows.

\begin{defn}[Hitting and Shattering]
Let $\C$ denote a set of curves in $\R^p$. 
We say that a subset of $\C$ is {\it hit} by a curve $s$ if the subset is exactly the set of curves in $\C$ which intersect the curve $s$. A collection of curves $\cS$ shatters the set $\C$ if for each subset $C$ of $\C$, there is some element $s$ of $\cS$ such that $s$ hits $C$.
\end{defn}

To extend our learning results to learning graphs built from several metrics, we will now state and prove a couple theorems involving algebraic hypersurfaces. Our results generalize significantly the techniques from \cite{dick2020semi} by bringing in novel connections with algebraic geometry.

\thmalghyp*

\begin{proof}
Let $\C$ denote a collection of $k$ algebraic hypersurfaces of degree at most $d$ in $\R^p$. 
We say that a subset of $\C$ is {\it hit} by a line segment if the subset is exactly the set of curves in $\C$ which intersect the segment, and {\it hit} by a line if some segment of the line hits the subset. We seek to upper bound the number of subsets of $\C$ which may be hit by axis-aligned line segments. We will first consider shattering by line segments in a fixed axial direction $x$. We can easily extend this to axis-aligned segments by noting they may hit only $p$ times as many subsets.


Let $L_c$ be a line in the $x$ direction. The subsets of $\C$ which may be hit by (segments along) $L_c$ is determined by the pattern of intersections of $L_c$ with hypersurfaces in $\C$. By Bezout's theorem, there are at most $kd+1$ distinct regions of $L_c$ due to the intersections. Therefore at most $\binom{kd+1}{2}$ distinct subsets may be hit.

Define the equivalence relation $L_{c_1} \sim L_{c_2}$
if the same hypersurfaces in $\C$ intersect $L_{c_1}$
and $L_{c_2}$, and in the same order (including with multiplicities). To determine these equivalence classes, we will project the hypersurfaces in $\C$ on to a hyperplane orthogonal to the $x$-direction. By the Tarski-Seidenberg-Łojasiewicz Theorem, we get a semi-algebraic collection $\C_x$, i.e. a set of polynomial equations and constraints in the projection space. Each cell of $\C_x$ corresponds to an equivalence class. Using well-known upper bounds for {\it cylindrical algebraic decomposition} (see for example \cite{england2016complexity}), we get that the number of equivalence classes is at most $O\left((2d)^{2^p-1}k^{2^p-1}2^{2^{p-1}}\right)$.

Putting it all together, the number of subsets hit by any axis aligned segment is at most $$O\left(p\binom{kd+1}{2}(2d)^{2^p-1}k^{2^p-1}2^{2^{p-1}}\right).$$

We are done as this is less than $2^k$ for fixed $d$ and $p$ and large enough $k$, and therefore all subsets may not be hit.

\end{proof}

\thmvcgeneral*

\begin{proof}
The proof is similar to that of Theorem \ref{thm:VC-bound} (see \cite{dick2020semi}). The main difference is that instead of relating the number
of ways intervals can label vectors of discontinuity points to the VC-dimension of intervals, we instead relate the
number of ways line segments can label vectors of $K$ algebraic hypersurfaces of degree $d$ to the VC-dimension of line
segments (when labeling algebraic hypersurfaces), which from Theorem \ref{thm:alg-hyp} is constant. To verify dispersion,
we need a uniform-convergence bound on the number of Lipschitz failures between the worst pair of points $\alpha,\alpha'$
at
distance $\le \epsilon$, but the definition allows us to bound the worst rate of discontinuties along any path between $\alpha,\alpha'$ of our
choice. We can bound the VC dimension of axis aligned segments against bounded-degree algebraic
hypersurfaces, which will allow us to establish dispersion by considering piecewise axis-aligned paths between points $\alpha$ and $\alpha'$.

Let $\C$ denote the set of all algebraic hypersurfaces of degree $d$. For simplicity, we assume that every function has its discontinuities specified by a collection of exactly $K$ algebraic hypersurfaces. For each function $l_t$, let $\gamma^{(t)}
\in \C^K$
denote the ordered tuple of algebraic hypersurfaces in $\C$ whose entries are the discontinuity locations of $l_t$. That is, $l_t$ has discontinuities along $(\gamma^{(t)}_1,\dots,\gamma^{(t)}_K)$,
but is otherwise $L$-Lispchitz. 

For any axis aligned path $s$, define the function $f_s : \C^K \rightarrow \{0, 1\}$ by
\begin{align*}
    f_s(\gamma) = \begin{cases*}
    1 &if for some $i \in [K]$ $\gamma_i$ intersects $s$\\
    0 & otherwise,
    \end{cases*}
\end{align*}
where $\gamma = (\gamma_1, \dots, \gamma_K) \in \C^K$. The sum $\sum_{t=1}^T f_s (\gamma^{(t)})$ counts the number of vectors $(\gamma^{(t)}_1,\dots,\gamma^{(t)}_K)$
that intersect $s$ or, equivalently, the number of functions $l_1, \dots , l_T$ that are not $L$-Lipschitz on $s$. We will
apply VC-dimension uniform convergence arguments to the class $\F = \{f_s : \C^K\rightarrow \{0, 1\} \mid s \text{ is an axis-aligned path}\}$.
In particular, we will show that for an independent set of vectors $(\gamma^{(t)}_1,\dots,\gamma^{(t)}_K)$, with high probability we have that $\frac{1}{T}\sum_{t=1}^T f_s (\gamma^{(t)})$ is close to $\E[\frac{1}{T}\sum_{t=1}^T f_s (\gamma^{(t)})]$ for all paths $s$. This uniform convergence argument will lead to the desired bounds.

Indeed, Theorem \ref{thm:alg-hyp} implies that VC dimension of $\F$ is $O(\log K)$. Now standard VC-dimension uniform convergence arguments for the class $\F$ imply that with probability at least $1-\delta$, for all $f_s\in\F$
\begin{align*}
    \left\lvert\frac{1}{T}\sum_{t=1}^T f_s (\gamma^{(t)})-\E\left[\frac{1}{T}\sum_{t=1}^T f_s (\gamma^{(t)})\right]\right\rvert\le O\left(\sqrt{\frac{\log(K/\delta)}{T}}\right)&\text{, or}\\
    \left\lvert\sum_{t=1}^T f_s (\gamma^{(t)})-\E\left[\sum_{t=1}^T f_s (\gamma^{(t)})\right]\right\rvert\le O\left(\sqrt{T\log(K/\delta)}\right).
\end{align*}
Now since $D(T,s)=\sum_{t=1}^Tf_s (\gamma^{(t)})$, we have for all $s$ and $\delta$, with probability at least $1-\delta$,
$\sup_{s\in L} D(T, s) \le \sup_{s\in L} \E[D(T, s)] +
O(\sqrt{T \log(K/\delta)})$. Taking expectation and setting $\delta=1/\sqrt{T}$ completes the proof as it allows us to bound the expected discontinuities by $O(\sqrt{T})$ when the above high probability event fails.
\end{proof}

Theorem \ref{thm:VC-bound-general} above generalizes the second step of the dispersion tool from single parameter families to several hyperparameters, and uses Theorem \ref{thm:alg-hyp} as a key ingredient. To complete the first step of in the multi-parameter setting, we can use a simple generalization of Theorem \ref{thm:poly-roots} by showing that few zeros are likely to occur on a piecewise axis-aligned path on whose pieces the zero sets of the
multivariate polynomial is the zero set of a single-variable
polynomial. Putting together we get Theorem \ref{thm:mutlimetric}.

\thmmultimetric*
\begin{proof}
Notice that $w(u,v)$ is a homogeneous polynomial in $\rho=(\rho_i,i\in[p])$. Further, the solutions of the quadratic objective subject to $f(u)=1/2$ for some $u$ are also homogeneous polynomial equations, of degree $nd$. Now to show dispersion, consider an axis-aligned path between any two parameter vectors $\rho,\rho'$ such that $\norm{\rho-\rho'}<\epsilon$ (notice that the definition of dispersion allows us to use any path between $\rho,\rho'$ for counting discontinuities). To compute the expected number of non-Lipchitzness in along this path, notice that for any fixed segment of this path, all but one variable are constant and the discontinuities are the zeros of single variable polynomial with bounded-density random coefficients, and that Theorem \ref{thm:poly-roots} applies. Summing along these paths we get at most $\Tilde{O}(p\epsilon)$ discontinuities in expectation for any $\norm{\rho-\rho'}<\epsilon$. Theorem \ref{thm:VC-bound-general} now completes the proof of dispersion in this case.
\end{proof}

\subsection{Semi-bandit efficient algorithms}
\label{app:sb}
In this appendix we present details of the efficient algorithms for computing the semi-bandit feedback sets in Algorithm \ref{alg:ddsslsb}. For unweighted graphs, we only have a polynomial number $O(n^2)$ of feedback sets and the feedback set for a given $\rho_t$ is readily computed by looking up a sorted list of distances $d(u,v)_{u,v\in L_i\cup U_i}$. For the weighted graph setting, we need non-trivial algorithms as discussed in Section \ref{sec: semibandit}.

\subsubsection{Min-cut objective}
First some notation for this section. We will use $G=(V,E)$ to denote an undirected graph with $V$ as the set of nodes  and $E\subseteq V\times V$ the weighted edges with capacity $d : E \rightarrow \R_{\ge 0}$. We are given special nodes $s,t\in V$ called {\it source} and {\it target} vertices. Recall the following definitions.

\begin{defn}
{\bf (s,t)-flows} An (s,t)-flow (or just flow if the source and target are clear from context) is a function
$f : V\times V \rightarrow \R_{\ge 0}$
that satisfies the conservation constraint at every vertex v except possibly
s and t given by $\sum_{(u,v)\in E}f(u,v)=\sum_{(v,u)\in E}f(v,u)$. The value of flow (also refered by just flow when clear from context) is the total flow out of $s$, $\sum_{u\in V}f(s,u)-\sum_{u\in V}f(u,s)$. 
\end{defn}

\begin{defn}
{\bf (s,t)-cut} An (s,t)-cut (or just a cut if the source and target are clear from context) is a partition of $V$ into $S,T$ such that $s\in S,t\in T$. We will denote the set $\{(u,v)\in E\mid u\in S,v\in T\}$ of edges in the cut by $\partial S$ or $\partial T$. The capacity of the cut is the total capacity of edges in the cut.
\end{defn}

For convenience we also define 

\begin{defn}\label{defn:pathflow}
Path flow. An (s,t)-flow is a path flow along a path $p=(s=v_0,v_1,\dots,v_n=t)$ if $f(u,w)>0$ iff $(u,w)=(v_i,v_{i+1})$ for some $i\in[n-1]$.
\end{defn}

\begin{defn}\label{defn:residual}
Residual capacity graph. Given a set of path flows $F$, the residual capacity graph (or simply the residual graph) is the graph $G'=(V,E)$ with capacities given by $c'(e)=c(e)-\sum_{f\in F}f(e)$.
\end{defn}

We will list without proof some well-known facts about maximum flows and minimum cuts in a graph which will be useful in our arguments.

\begin{fact}
\begin{itemize}[leftmargin=0.5cm]
\item[1.] Let $f$ be any feasible $(s, t)$-flow, and let $(S, T)$ be any $(s, t)$-cut. The value of $f$ is at
most the capacity of $(S, T)$. Moreover, the value of $f$ equals the capacity of $(S,T)$ if and only if $f$ saturates every edge in the cut.
\item[2.] Max-flow min-cut theorem. The value of maximum (value of) $(s, t)$-flow equals the capacity of the minimum $(s, t)$-cut. It may be computed in $O(VE)$ time.
\item[3.] Flow Decomposition Theorem. Every feasible $(s, t)$-flow $f$ can be written as a weighted sum of
directed $(s, t)$-paths and directed cycles. Moreover, a directed edge $(u,v)$ appears in at least one of
these paths or cycles if and only if $f (u,v) > 0$, and the total number of paths and cycles is at most
the number of edges in the network. It may be computed in $O(VE)$ time.
\end{itemize}
\end{fact}

We now have the machinery to prove the correctness and analyze the time complexity of our Algorithm \ref{algorithm: semi cut}.


\thmsb*

\begin{proof}
First, we briefly recall the set up of the mincut objective. Let $L_1$ and $L_2$ denote the labeled points $L$ of different classes. To obtain the labels for $U$, we seek the smallest cut $(V_1, V\setminus V_1)$ of $G$ separating the nodes in $L_1$ and $L_2$. To frame as $s,t$-cut we can augment the data graph with nodes $s,t$, and add infinite capacity edges to nodes in $L_1$ and $L_2$ respectively. If $L_i\subseteq V_1$, label exactly the nodes in $V_1$ with label $i$.  The loss function, $l(\sigma)$ gives the fraction of labels this procedure gets right for the unlabeled set $U$. 

If the min-cut is the same for two values of $\sigma$, then so is prediction on each point and thus the loss function $l(\sigma)$. So we seek the smallest amount of change in $\sigma$ so that the mincut changes. Our semi-bandit feedback set is given by the intervals for which the min-cut is fixed. Consider a fixed value of $\sigma=\sigma_0$ and the corresponding graph $G(\sigma_0)$. We can compute the max-flow on $G(\sigma_0)$, and simultaneously obtain a min-cut $(V_1,V \setminus V_1)$ in time $O(VE)=O(n^3)$. All the edges in $\partial V_1$ are saturated by the flow. Obtain the flow decomposition of the max-flow (again $O(VE)=O(n^3)$). For each $e_i\in \partial V_1$, let $f_i$ be a path flow through $e_i$ from the flow decomposition (cycle flows cannot saturate, or even pass through, $e_i$ since it is on the min-cut). Note that the $f_i$ are distinct due to the max-flow min-cut theorem. Now as $\sigma$ is increased, we increment each $f_i$ by the additional capacity in the corresponding edge $e_i$, until an edge $e'$ in $E\setminus \partial V_1$ saturates (at a faster rate than the flow through it). This can be detected by expressing $f_i$ as a function of $\sigma$ for each $f_i$ and computing the zero of an exponential polynomial capturing the change in residual capacity of any edge $e\notin \partial V_1$. Let $f_j$ be one of the path flows through $e'$. We reassign this flow to $e'$ (it will now increase with $e'$ as its bottleneck) and find an alternate path avoiding this edge through non-saturated edges and $e_j$ (if one exists) along which we send the new $f_j$.  We now increment all the path flows as before keeping their bottleneck edges saturated. The procedure stops when we can no longer find an alternate path for some $e_j$. But this means we must have a new cut with the saturated edges, and therefore a new min-cut. This gives us a new critical value of $\sigma$, and the desired upper end for the feedback interval. Obtaining the lower end is possible by a symmetric procedure that decreases the path flows while keeping edges saturated.

We remark that our procedure differs from the well-known algorithms for obtaining min-cuts in a static graph. The greedy procedures for static graphs need directed edges $(u,v)$ and $(v,u)$ in the residual graph, and find paths through unsaturated edges through this graph to increase the flow, and cannot work with monotonically increasing path flows. We however start with a max flow and maintain the invariant that our flow equals some cut size throughout.

Finally note that each time we perform step 9 of the algorithm, a new saturated edge stays saturated for all further $\sigma$ until the new cut is found. So we can do this at most $O(n^2)$ times. In each loop we need to obtain the saturation condition for $O(n)$ edges corresponding to one new path flow.
\end{proof}

We remind the reader that a remarkable property of finding the min-cuts dynamically in our setting is an interesting ``hybrid" combinatorial and continuous set-up, which may be of independent interest. A similar dynamic, but purely combinatorial, setting for recomputing flows efficiently online over a discrete graph sequence has been studied in \cite{altner2008rapidly}.


\section{Distributional setting}\label{app:distrib}
In this appendix we include details of proofs and algorithms from section \ref{sec: distrib}. Recall that we define the set of loss functions $\mathcal{H}_r=\{l_{A(G(r),L,U)}\mid 0\le r<\infty \}$, where $G(r)$ is the family of threshold graphs specified by Definition \ref{defn:g}a, and $\mathcal{H}_{\sigma}=\{l_{A(G(\sigma),L,U)}\mid 0\le \sigma<\infty \}$, where $G(\sigma)$ is the family of exponential kernel graphs specified by Definition \ref{defn:g}c. We show lower bounds on the pseudodimension of these function classes below.

\thmpdlb*

We first prove the following
useful statement which helps us construct general examples with desirable properties. In particular,
the following lemma guarantees that given a sequence of values of $r$ of size $O(n)$, it is possible to
construct an instance $S$ of partially labeled points such that the cost of the output of algorithm $A(G(r),L)$ on V as a function of $r$ oscillates above and below some threshold as $r$ moves along the sequence of intervals
$(r_i,r_{i+1})$. Given this powerful guarantee, we can then pick appropriate sequences of $r$ and generate
a sample set of $\Omega(\log n)$ instances that correspond to cost functions that oscillate in a manner that
helps us pick $\Omega(n)$ values of $r$ that shatters the samples. 

\begin{lemma}\label{lem:pdlb}
Given integer $n>5$ and a sequence of $n'$ $r$'s such that $1<r_1<r_2<\dots<r_{n'}<2$ and $n'\le n-5$, there exists a real valued witness $w > 0$ and
a labeling instance $S$ of partially labeled $n$ points, 
such that for $0 \le i \le n'/2 - 1$,
$l_{A(G(r),L)} < w$ for
$r \in (r_{2i}, r_{2i+1})$,
and $l_{A(G(r),L)} > w$ for
$r \in (r_{2i+1}, r_{2i+2})$ (where $r_0$ and $r_{n'+1}$ correspond to immediate left and right neighborhoods respectively of $r_1$ and $r_{n'}$).
\end{lemma}
\begin{proof}
We first present a sketch of the construction. We will use binary labels $a$ and $b$. We further have three points labeled $a$ (namely $a_1,a_2,a_3$) and two points labeled $b$ (say $b_1,b_2$). At some intial $r=r_0$, all the like-labeled points are connected in $G(r_0)$ and all the unlabeled points (namely $u_1,\dots,u_{n'}$) are connected to $a_1$ as shown in Figure \ref{fig:lb1}. The algorithm $A(G(r),L)$ labels everything $a$ and gets exactly half the labels right. As $r$ is increased to $r_i$, $u_i$ gets connected to $b_1$ and $b_2$ (Figure \ref{fig:lb2}). If the sequence $u_i$ is alternately labeled, the loss increases and decreases alternately as all the predicted labels turn to $b$ as $r$ is increased to $r_{n'}$. Further increasing $r$ may connect all the unlabeled points with true label $a$ to $a_2$ and $a_3$ (Figure \ref{fig:lb3}), although this is not crucial to our argument.
The rest of the proof gives concrete values of $r$ and verifies that the construction is indeed feasible. 

We will ensure all the pairwise distances are between 1 and 2, so that triangle inequality is always satisfied. It may also be readily verified that $O(\log n)$ dimensions suffice for our construction to exist. We start by defining some useful constants. We pick $r_-,r_{+},r_{\max}\in (1,2)$ such that $r_-<r_1<\dots<r_{n'}<r_{+}<r_{\max}$,
\begin{align*}
    r_-&= \frac{1+r_1}{2},\\
    r_{+}&= 1+\frac{r_{n'}}{2},\\
    r_{\max}&= 1+\frac{r_{+}}{2}.
\end{align*}
We will now specify the distances of the labeled points. The points with the same label are close together and away from the oppositely labeled points.
\begin{align*}
    d(a_i,a_j)&= r_-,&&1\le i<j \le 3,\\
    d(b_1,b_2)&= r_-,&&\\
    d(a_i,b_j)&= r_{\max},&&1\le i \le 3,1\le j \le 2.
\end{align*}
Further, the unlabeled points are located as follows
\begin{align*}
    d(a_1,u_k)&= r_-,&&1\le k \le n',\\
    d(b_i,u_k)&= r_k,&&1\le k \le n',1\le i \le 2,\\
    d(a_i,u_k)&= r_{+},&&1\le k \le n',2\le i \le 3,\\
    d(u_i,u_j)&= r_{\max},&&1\le i<j \le n'.
\end{align*}
That is, all unknown points are closest to $a_1$, followed by $b_i$'s, remaining $a_i$'s and other $u_i$'s in order.
Further let the true labels of the unlabeled nodes be alternating with the index, i.e. $u_k$ is $a$ if and only if $k$ is even.

We will now compute the loss for the soft labeling algorithm $A(G(r),L)$ of \cite{zhu2003semi} as $r$ varies from $r_-$ to $r_+$, starting with $r=r_0=r_-$. We note that our construction also works for other algorithms as well, for example the min-cut based approach of  \cite{blum2001learning}, but omit the details.

For the graph $G(r_-)$, $A(G,L)$ labels each unknown node as $a$ since each unknown point is a leaf node connected to $a_1$. Indeed if $f(a_1)=1$, the quadratic objective attains the minimum of 0 for exactly $f(u_k)=1$ for each $1\le k\le n'$.  This results in half the labels in the dataset being incorrectly labeled since we stipulate that half the unknown labels are of each category. This results in loss $l_{A(G(r_-),L)}=:l_{\text{high}}$ say.

\begin{figure}[t]
    \centering
    
    \begin{tabular}{c }
    
    \begin{subfigure}[b]{0.3\textwidth}

\tikzset{every picture/.style={line width=0.75pt}} 

\begin{tikzpicture}[x=0.75pt,y=0.75pt,yscale=-1,xscale=1]

\draw   (476.44,357.05) .. controls (476.44,355.1) and (478.02,353.52) .. (479.97,353.52) .. controls (481.92,353.52) and (483.5,355.1) .. (483.5,357.05) .. controls (483.5,359) and (481.92,360.58) .. (479.97,360.58) .. controls (478.02,360.58) and (476.44,359) .. (476.44,357.05) -- cycle ;
\draw   (517.97,361.59) .. controls (517.97,359.64) and (519.55,358.06) .. (521.5,358.06) .. controls (523.45,358.06) and (525.03,359.64) .. (525.03,361.59) .. controls (525.03,363.54) and (523.45,365.13) .. (521.5,365.13) .. controls (519.55,365.13) and (517.97,363.54) .. (517.97,361.59) -- cycle ;
\draw    (478.5,378.05) -- (519.97,364.59) ;
\draw    (483.5,357.05) -- (517.97,361.59) ;
\draw   (472.97,379.58) .. controls (472.97,377.63) and (474.55,376.05) .. (476.5,376.05) .. controls (478.45,376.05) and (480.03,377.63) .. (480.03,379.58) .. controls (480.03,381.53) and (478.45,383.11) .. (476.5,383.11) .. controls (474.55,383.11) and (472.97,381.53) .. (472.97,379.58) -- cycle ;
\draw    (479.97,360.58) -- (476.5,376.05) ;
\draw    (522.5,365.13) -- (556.5,410.05) ;
\draw   (555.5,412.05) .. controls (555.5,410.1) and (557.08,408.52) .. (559.03,408.52) .. controls (560.98,408.52) and (562.56,410.1) .. (562.56,412.05) .. controls (562.56,414) and (560.98,415.58) .. (559.03,415.58) .. controls (557.08,415.58) and (555.5,414) .. (555.5,412.05) -- cycle ;
\draw  [color={rgb, 255:red, 0; green, 0; blue, 0 }  ,draw opacity=1 ][line width=3] [line join = round][line cap = round] (557,359.06) .. controls (557,359.06) and (557,359.06) .. (557,359.06) ;
\draw  [color={rgb, 255:red, 0; green, 0; blue, 0 }  ,draw opacity=1 ][line width=3] [line join = round][line cap = round] (557,371.06) .. controls (557,371.06) and (557,371.06) .. (557,371.06) ;
\draw  [color={rgb, 255:red, 0; green, 0; blue, 0 }  ,draw opacity=1 ][line width=3] [line join = round][line cap = round] (557,382.06) .. controls (557,382.06) and (557,382.06) .. (557,382.06) ;
\draw    (524.5,360.06) -- (559.03,321.58) ;
\draw   (557.5,320.05) .. controls (557.5,318.1) and (559.08,316.52) .. (561.03,316.52) .. controls (562.98,316.52) and (564.56,318.1) .. (564.56,320.05) .. controls (564.56,322) and (562.98,323.58) .. (561.03,323.58) .. controls (559.08,323.58) and (557.5,322) .. (557.5,320.05) -- cycle ;
\draw    (521.5,358.06) -- (557.03,284.58) ;
\draw   (555.5,283.05) .. controls (555.5,281.1) and (557.08,279.52) .. (559.03,279.52) .. controls (560.98,279.52) and (562.56,281.1) .. (562.56,283.05) .. controls (562.56,285) and (560.98,286.58) .. (559.03,286.58) .. controls (557.08,286.58) and (555.5,285) .. (555.5,283.05) -- cycle ;
\draw   (602.44,356.05) .. controls (602.44,354.1) and (604.02,352.52) .. (605.97,352.52) .. controls (607.92,352.52) and (609.5,354.1) .. (609.5,356.05) .. controls (609.5,358) and (607.92,359.58) .. (605.97,359.58) .. controls (604.02,359.58) and (602.44,358) .. (602.44,356.05) -- cycle ;
\draw    (605.97,359.58) -- (602.5,375.05) ;
\draw   (596.97,378.58) .. controls (596.97,376.63) and (598.55,375.05) .. (600.5,375.05) .. controls (602.45,375.05) and (604.03,376.63) .. (604.03,378.58) .. controls (604.03,380.53) and (602.45,382.11) .. (600.5,382.11) .. controls (598.55,382.11) and (596.97,380.53) .. (596.97,378.58) -- cycle ;

\draw (518,345) node   [align=left] {$\displaystyle a_{1}$};
\draw (470,343) node   [align=left] {$\displaystyle a_{2}$};
\draw (465,384) node   [align=left] {$\displaystyle a_{3}$};
\draw (563,426) node   [align=left] {$\displaystyle u_{n'}$};
\draw (563,332) node   [align=left] {$\displaystyle u_{2}$};
\draw (561,295) node   [align=left] {$\displaystyle u_{1}$};
\draw (596,342) node   [align=left] {$\displaystyle b_{1}$};
\draw (594,388) node   [align=left] {$\displaystyle b_{2}$};

\end{tikzpicture}

 \subcaption{$G(r_-)$}
    \label{fig:lb1}
    \end{subfigure}
    
    \hfill
    \begin{subfigure}[b]{0.3\textwidth}

\tikzset{every picture/.style={line width=0.75pt}} 

\begin{tikzpicture}[x=0.75pt,y=0.75pt,yscale=-1,xscale=1]

\draw   (476.44,357.05) .. controls (476.44,355.1) and (478.02,353.52) .. (479.97,353.52) .. controls (481.92,353.52) and (483.5,355.1) .. (483.5,357.05) .. controls (483.5,359) and (481.92,360.58) .. (479.97,360.58) .. controls (478.02,360.58) and (476.44,359) .. (476.44,357.05) -- cycle ;
\draw   (517.97,361.59) .. controls (517.97,359.64) and (519.55,358.06) .. (521.5,358.06) .. controls (523.45,358.06) and (525.03,359.64) .. (525.03,361.59) .. controls (525.03,363.54) and (523.45,365.13) .. (521.5,365.13) .. controls (519.55,365.13) and (517.97,363.54) .. (517.97,361.59) -- cycle ;
\draw    (478.5,378.05) -- (519.97,364.59) ;
\draw    (483.5,357.05) -- (517.97,361.59) ;
\draw   (472.97,379.58) .. controls (472.97,377.63) and (474.55,376.05) .. (476.5,376.05) .. controls (478.45,376.05) and (480.03,377.63) .. (480.03,379.58) .. controls (480.03,381.53) and (478.45,383.11) .. (476.5,383.11) .. controls (474.55,383.11) and (472.97,381.53) .. (472.97,379.58) -- cycle ;
\draw    (479.97,360.58) -- (476.5,376.05) ;
\draw    (522.5,365.13) -- (556.5,410.05) ;
\draw   (555.5,412.05) .. controls (555.5,410.1) and (557.08,408.52) .. (559.03,408.52) .. controls (560.98,408.52) and (562.56,410.1) .. (562.56,412.05) .. controls (562.56,414) and (560.98,415.58) .. (559.03,415.58) .. controls (557.08,415.58) and (555.5,414) .. (555.5,412.05) -- cycle ;
\draw  [color={rgb, 255:red, 0; green, 0; blue, 0 }  ,draw opacity=1 ][line width=3] [line join = round][line cap = round] (559,320.06) .. controls (559,320.06) and (559,320.06) .. (559,320.06) ;
\draw  [color={rgb, 255:red, 0; green, 0; blue, 0 }  ,draw opacity=1 ][line width=3] [line join = round][line cap = round] (559,329.06) .. controls (559,329.06) and (559,329.06) .. (559,329.06) ;
\draw  [color={rgb, 255:red, 0; green, 0; blue, 0 }  ,draw opacity=1 ][line width=3] [line join = round][line cap = round] (559,338.06) .. controls (559,338.06) and (559,338.06) .. (559,338.06) ;
\draw    (521.5,358.06) -- (559.03,303.58) ;
\draw   (557.5,302.05) .. controls (557.5,300.1) and (559.08,298.52) .. (561.03,298.52) .. controls (562.98,298.52) and (564.56,300.1) .. (564.56,302.05) .. controls (564.56,304) and (562.98,305.58) .. (561.03,305.58) .. controls (559.08,305.58) and (557.5,304) .. (557.5,302.05) -- cycle ;
\draw   (602.44,356.05) .. controls (602.44,354.1) and (604.02,352.52) .. (605.97,352.52) .. controls (607.92,352.52) and (609.5,354.1) .. (609.5,356.05) .. controls (609.5,358) and (607.92,359.58) .. (605.97,359.58) .. controls (604.02,359.58) and (602.44,358) .. (602.44,356.05) -- cycle ;
\draw    (605.97,359.58) -- (602.5,375.05) ;
\draw   (596.97,378.58) .. controls (596.97,376.63) and (598.55,375.05) .. (600.5,375.05) .. controls (602.45,375.05) and (604.03,376.63) .. (604.03,378.58) .. controls (604.03,380.53) and (602.45,382.11) .. (600.5,382.11) .. controls (598.55,382.11) and (596.97,380.53) .. (596.97,378.58) -- cycle ;
\draw    (605.97,352.52) -- (562.56,307.11) ;
\draw    (600.5,375.05) -- (562.56,307.11) ;
\draw    (560.56,364.05) -- (602.44,356.05) ;
\draw    (560.56,364.05) -- (596.97,378.58) ;
\draw  [color={rgb, 255:red, 0; green, 0; blue, 0 }  ,draw opacity=1 ][line width=3] [line join = round][line cap = round] (558,380.06) .. controls (558,380.06) and (558,380.06) .. (558,380.06) ;
\draw  [color={rgb, 255:red, 0; green, 0; blue, 0 }  ,draw opacity=1 ][line width=3] [line join = round][line cap = round] (558,390.06) .. controls (558,390.06) and (558,390.06) .. (558,390.06) ;
\draw  [color={rgb, 255:red, 0; green, 0; blue, 0 }  ,draw opacity=1 ][line width=3] [line join = round][line cap = round] (558,399.06) .. controls (558,399.06) and (558,399.06) .. (558,399.06) ;
\draw   (553.5,366.05) .. controls (553.5,364.1) and (555.08,362.52) .. (557.03,362.52) .. controls (558.98,362.52) and (560.56,364.1) .. (560.56,366.05) .. controls (560.56,368) and (558.98,369.58) .. (557.03,369.58) .. controls (555.08,369.58) and (553.5,368) .. (553.5,366.05) -- cycle ;
\draw    (525.03,361.59) -- (553.5,364.05) ;

\draw (520,372) node   [align=left] {$\displaystyle a_{1}$};
\draw (470,343) node   [align=left] {$\displaystyle a_{2}$};
\draw (465,384) node   [align=left] {$\displaystyle a_{3}$};
\draw (562,425) node   [align=left] {$\displaystyle u_{n'}$};
\draw (564,288) node   [align=left] {$\displaystyle u_{1}$};
\draw (608,334) node   [align=left] {$\displaystyle b_{1}$};
\draw (610,392) node   [align=left] {$\displaystyle b_{2}$};
\draw (559,352) node   [align=left] {$\displaystyle u_{i}$};

\end{tikzpicture}

 \subcaption{$G(r_i)$}
    \label{fig:lb2}
    \end{subfigure}
    
    \begin{subfigure}[b]{0.3\textwidth}
    
\tikzset{every picture/.style={line width=0.75pt}} 

\begin{tikzpicture}[x=0.75pt,y=0.75pt,yscale=-1,xscale=1]

\draw   (476.44,357.05) .. controls (476.44,355.1) and (478.02,353.52) .. (479.97,353.52) .. controls (481.92,353.52) and (483.5,355.1) .. (483.5,357.05) .. controls (483.5,359) and (481.92,360.58) .. (479.97,360.58) .. controls (478.02,360.58) and (476.44,359) .. (476.44,357.05) -- cycle ;
\draw   (517.97,361.59) .. controls (517.97,359.64) and (519.55,358.06) .. (521.5,358.06) .. controls (523.45,358.06) and (525.03,359.64) .. (525.03,361.59) .. controls (525.03,363.54) and (523.45,365.13) .. (521.5,365.13) .. controls (519.55,365.13) and (517.97,363.54) .. (517.97,361.59) -- cycle ;
\draw    (478.5,378.05) -- (519.97,364.59) ;
\draw    (483.5,357.05) -- (517.97,361.59) ;
\draw   (472.97,379.58) .. controls (472.97,377.63) and (474.55,376.05) .. (476.5,376.05) .. controls (478.45,376.05) and (480.03,377.63) .. (480.03,379.58) .. controls (480.03,381.53) and (478.45,383.11) .. (476.5,383.11) .. controls (474.55,383.11) and (472.97,381.53) .. (472.97,379.58) -- cycle ;
\draw    (479.97,360.58) -- (476.5,376.05) ;
\draw    (522.5,365.13) -- (556.5,410.05) ;
\draw   (555.5,412.05) .. controls (555.5,410.1) and (557.08,408.52) .. (559.03,408.52) .. controls (560.98,408.52) and (562.56,410.1) .. (562.56,412.05) .. controls (562.56,414) and (560.98,415.58) .. (559.03,415.58) .. controls (557.08,415.58) and (555.5,414) .. (555.5,412.05) -- cycle ;
\draw  [color={rgb, 255:red, 0; green, 0; blue, 0 }  ,draw opacity=1 ][line width=3] [line join = round][line cap = round] (559,352.06) .. controls (559,352.06) and (559,352.06) .. (559,352.06) ;
\draw  [color={rgb, 255:red, 0; green, 0; blue, 0 }  ,draw opacity=1 ][line width=3] [line join = round][line cap = round] (559,364.06) .. controls (559,364.06) and (559,364.06) .. (559,364.06) ;
\draw  [color={rgb, 255:red, 0; green, 0; blue, 0 }  ,draw opacity=1 ][line width=3] [line join = round][line cap = round] (559,375.06) .. controls (559,375.06) and (559,375.06) .. (559,375.06) ;
\draw    (521.5,358.06) -- (559.03,303.58) ;
\draw   (557.5,302.05) .. controls (557.5,300.1) and (559.08,298.52) .. (561.03,298.52) .. controls (562.98,298.52) and (564.56,300.1) .. (564.56,302.05) .. controls (564.56,304) and (562.98,305.58) .. (561.03,305.58) .. controls (559.08,305.58) and (557.5,304) .. (557.5,302.05) -- cycle ;
\draw   (602.44,356.05) .. controls (602.44,354.1) and (604.02,352.52) .. (605.97,352.52) .. controls (607.92,352.52) and (609.5,354.1) .. (609.5,356.05) .. controls (609.5,358) and (607.92,359.58) .. (605.97,359.58) .. controls (604.02,359.58) and (602.44,358) .. (602.44,356.05) -- cycle ;
\draw    (605.97,359.58) -- (602.5,375.05) ;
\draw   (596.97,378.58) .. controls (596.97,376.63) and (598.55,375.05) .. (600.5,375.05) .. controls (602.45,375.05) and (604.03,376.63) .. (604.03,378.58) .. controls (604.03,380.53) and (602.45,382.11) .. (600.5,382.11) .. controls (598.55,382.11) and (596.97,380.53) .. (596.97,378.58) -- cycle ;
\draw    (481.97,357.58) -- (556.5,410.05) ;
\draw    (480.03,379.58) -- (556.5,410.05) ;
\draw    (483.5,357.05) -- (559.03,303.58) ;
\draw    (476.5,376.05) -- (559.03,303.58) ;
\draw    (605.97,352.52) -- (562.56,307.11) ;
\draw    (600.5,375.05) -- (562.56,307.11) ;
\draw    (559.03,408.52) -- (602.44,356.05) ;
\draw    (559.03,408.52) -- (596.97,378.58) ;

\draw (520,372) node   [align=left] {$\displaystyle a_{1}$};
\draw (470,343) node   [align=left] {$\displaystyle a_{2}$};
\draw (465,384) node   [align=left] {$\displaystyle a_{3}$};
\draw (562,425) node   [align=left] {$\displaystyle u_{n'}$};
\draw (563,320) node   [align=left] {$\displaystyle u_{1}$};
\draw (608,334) node   [align=left] {$\displaystyle b_{1}$};
\draw (610,392) node   [align=left] {$\displaystyle b_{2}$};

\end{tikzpicture}

 \subcaption{$G(r_+)$}
    \label{fig:lb3}
    \end{subfigure}
    
    \end{tabular}

 \caption{Graphs $G(r)$ as $r$ is varied, for lower bound construction for pseudodimension of $\mathcal{H}_r$.}
    \label{fig:lb}
\end{figure}
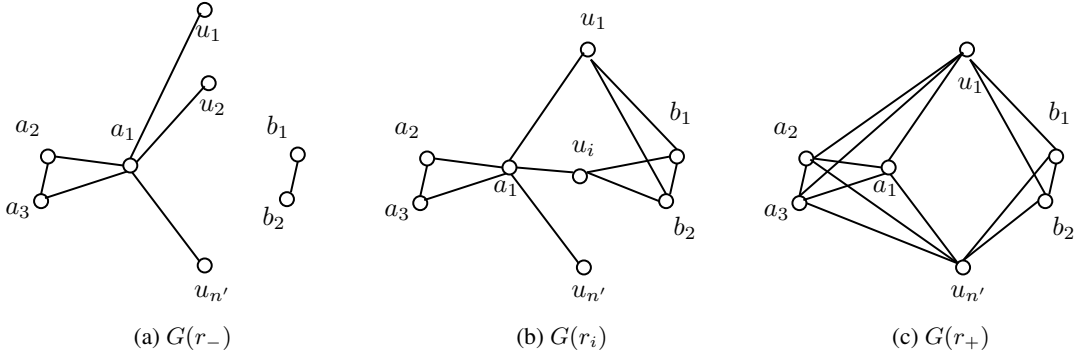

Now as $r$ is increased to $r_1$, the edges $(b_i,u_1)$, $i=1,2$ are added with $b_i$ labeled as $f(b_i)=0$. This results in a fractional label of $\frac{1}{3}$ for $f(u_1)$ while $f(u_k)=1$ for $k\ne 1$. Indeed the terms involving $f(u_1)$ in the objective are $(1-f(u_1))^2+2f(u_1)^2$, which is minimized at $\frac{1}{3}$. Since $u_1$ has true label $b$, this results in a slightly smaller loss of $l_{A(G(r_1),L)}=:l_{\text{low}}$. This happens when $A$ uses rounding, or in expectation if $A$ uses randomized prediction with probability $f(u)$.

At the next critical point $r_2$, $u_2$  gets connected to $b_i$'s and gets incorrectly classified as $b$. This increases the loss again to $l_{\text{high}}$. The loss function thus alternates as $r$ is varied through the specified values, between $l_{\text{high}}$ and $l_{\text{low}}$. We therefore set the witness $w$ to something in between.
$$w=\frac{l_{\text{low}}+l_{\text{high}}}{2}.$$
\end{proof}

\begin{proof}[Proof of Theorem \ref{thm:pdlb}] We will now use Lemma \ref{lem:pdlb} to prove our lower bound.  Arbitarily choose $n'=n-5$ (assumed to be a power of 2 just for convenient presentation) real numbers $r_{[000\dots 01]}<r_{[000\dots 10]}<\dots <r_{[111\dots 11]}$ in $(1,2)$. The indices are increasing binary numbers of length $m=\log n'$. We create labeling instances using Lemma \ref{lem:pdlb} which can be shattered by these $r$ values. Instance $S_i=(G_i,L_i)$ corresponds to fluctuation of $i$-th bit $b_i$ in our $r_b$ sequence, where $b=(b_1,\dots,b_m)\in\{0,1\}^m$, i.e., we apply the lemma by using a subset of the $r_b$ values which correspond to the bit flips in the $i$-th binary digit. For example, $S_1$ just needs a single bit flip (at $r_{[100\dots 00]}$). The lemma gives us both the instances and corresponding witnesses $w_i$.

This construction ensures $\text{sign}(l_{A(G_i(r_b),L_i)}-w_i)=b_i$, i.e. the set of instances is shattered. Thus the pseudodimension is at least $\log (n-5)=\Omega(\log n)$.

\end{proof}

\thmpdlbs*

\begin{proof}
The plan for the proof is to first construct a graph where the edge weights are carefully selected, so that we have $2^N$ oscillations in the loss function with $\sigma$ for $N=\Omega(n)$. Then we use this construction to create $\Theta(n)$ instances, each having a subset of the oscillations so that each interval leads to a unique labeling of the instances, for a total of $2^N$ labelings, which would imply pseudodimension is $\Omega(n)$. We will present our discussion in terms of the min-cut objective, for simplicity of presentation.

{\it Graph construction}: First a quick rough overview.  We start with a pair of labeled nodes of each class, and a pair of unlabeled nodes which may be assigned either label depending on $\sigma$.  We then build the graph in $N=(n-4)/2$ stages, adding two new nodes at each step with carefully chosen distances from existing nodes. Before adding the $i$th pair $x_i,y_i$ of nodes, there will be 
$2^{i-1}$ intervals of $\sigma$ such that the intervals correspond to distinct min-cuts which result in all possible labelings of $\{x_1,\dots,x_{i-1}\}$. Moreover, $y_j$ will be labeled differently from $x_j$ in each of these intervals. The edges to the new nodes will ensure that the cuts that differ exactly in $x_i$ will divide each of these intervals giving us $2^i$ intervals where distinct mincuts give all labelings of $\{x_1,\dots,x_{i}\}$, and allowing an inductive proof. The challenge is that we only get to set $O(i)$ edges but seek properties about $2^{i}$ cuts, so we must do this carefully. 

 Let $\varsigma=e^{-1/\sigma^2}$. Notice $\varsigma\in(0,1)$, and bijectively corresponds to $\sigma\in(0,\infty)$ (due to monotonicity) and therefore it suffices to specify intervals of $\varsigma$ corresponding to different labelings. Further we can specify distances $d(u,v)$ between pairs of nodes $u,v$ by specifying the squared distance $d(u,v)^2$. For the remainder of this proof we will refer to $\delta(u,v)=d(u,v)^2$ by {\it distance} and set values in $[1.5,1.6]$. Consequently, $d(u,v)\in(1.22,1.27)$ and therefore the triangle inequality is always satisfied. Notice that with this notation, the graph weights will be $w(u,v)=\varsigma^{\delta(u,v)}$.
 
We now provide details of the construction. We have four labeled nodes as follows. $a_1,a_2$ are labeled 0 and are collectively denoted by $A=\{a_1,a_2\}$, similarly $b_1,b_2$ are labeled 1 and $B=\{b_1,b_2\}$. Note that edges between these nodes are on all or no cut separating $A,B$, we set the distances to 1.6 and call this graph $G_0$. We further add unlabeled nodes in pairs $(x_j,y_j)$ in {\it rounds} $1\le j\le N$. In round $i$, we construct graph $G_i$ by adding nodes $(x_i,y_i)$ to $G_{i-1}$. The distances are set to ensure that for  $G_N$ there are $2^N$ unique values of $\varsigma$ corresponding to distinct min-cuts, each giving a unique labeling for $\{x_1,\dots,x_n\}$ (and the complementary labeling for $\{y_1,\dots,y_n\}$). Moreover subsets of these points also obtain the unique labeling for $\{x_1,\dots,x_i\}$ for each $G_i$.

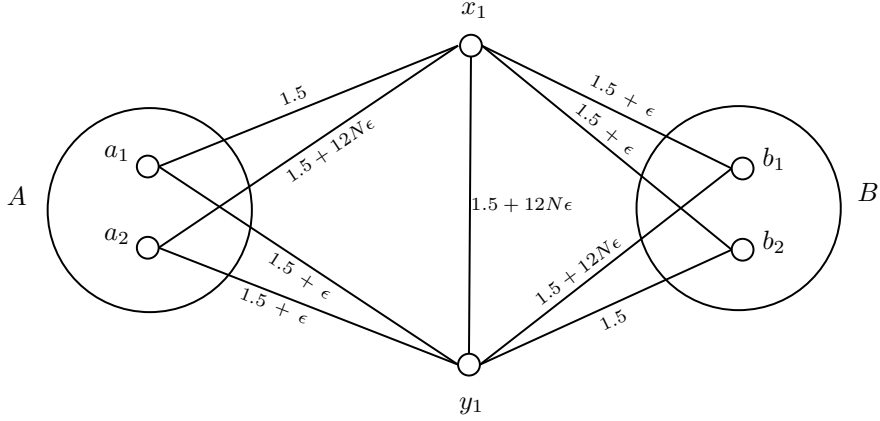
\begin{figure}[t]
    \centering
    \tikzset{every picture/.style={line width=0.75pt}} 

\begin{tikzpicture}[x=0.75pt,y=0.75pt,yscale=-1,xscale=1]

\draw   (154,140.5) .. controls (154,112.06) and (177.06,89) .. (205.5,89) .. controls (233.94,89) and (257,112.06) .. (257,140.5) .. controls (257,168.94) and (233.94,192) .. (205.5,192) .. controls (177.06,192) and (154,168.94) .. (154,140.5) -- cycle ;
\draw   (451,139.5) .. controls (451,111.06) and (474.06,88) .. (502.5,88) .. controls (530.94,88) and (554,111.06) .. (554,139.5) .. controls (554,167.94) and (530.94,191) .. (502.5,191) .. controls (474.06,191) and (451,167.94) .. (451,139.5) -- cycle ;
\draw   (362,57.5) .. controls (362,54.46) and (364.46,52) .. (367.5,52) .. controls (370.54,52) and (373,54.46) .. (373,57.5) .. controls (373,60.54) and (370.54,63) .. (367.5,63) .. controls (364.46,63) and (362,60.54) .. (362,57.5) -- cycle ;
\draw   (361,218.5) .. controls (361,215.46) and (363.46,213) .. (366.5,213) .. controls (369.54,213) and (372,215.46) .. (372,218.5) .. controls (372,221.54) and (369.54,224) .. (366.5,224) .. controls (363.46,224) and (361,221.54) .. (361,218.5) -- cycle ;
\draw   (199,118.5) .. controls (199,115.46) and (201.46,113) .. (204.5,113) .. controls (207.54,113) and (210,115.46) .. (210,118.5) .. controls (210,121.54) and (207.54,124) .. (204.5,124) .. controls (201.46,124) and (199,121.54) .. (199,118.5) -- cycle ;
\draw   (199,159.5) .. controls (199,156.46) and (201.46,154) .. (204.5,154) .. controls (207.54,154) and (210,156.46) .. (210,159.5) .. controls (210,162.54) and (207.54,165) .. (204.5,165) .. controls (201.46,165) and (199,162.54) .. (199,159.5) -- cycle ;
\draw   (499,119.5) .. controls (499,116.46) and (501.46,114) .. (504.5,114) .. controls (507.54,114) and (510,116.46) .. (510,119.5) .. controls (510,122.54) and (507.54,125) .. (504.5,125) .. controls (501.46,125) and (499,122.54) .. (499,119.5) -- cycle ;
\draw   (499,160.5) .. controls (499,157.46) and (501.46,155) .. (504.5,155) .. controls (507.54,155) and (510,157.46) .. (510,160.5) .. controls (510,163.54) and (507.54,166) .. (504.5,166) .. controls (501.46,166) and (499,163.54) .. (499,160.5) -- cycle ;
\draw    (367.5,63) -- (366.5,213) ;
\draw    (361,57.5) -- (210,118.5) ;
\draw    (361,57.5) -- (210,159.5) ;
\draw    (373,57.5) -- (499,119.5) ;
\draw    (373,57.5) -- (499,160.5) ;
\draw    (361,218.5) -- (210,118.5) ;
\draw    (361,218.5) -- (210,159.5) ;
\draw    (372,218.5) -- (499,160.5) ;
\draw    (372,218.5) -- (499,119.5) ;

\draw (132,127.4) node [anchor=north west][inner sep=0.75pt]    {$A$};
\draw (561,125.4) node [anchor=north west][inner sep=0.75pt]    {$B$};
\draw (360,233.4) node [anchor=north west][inner sep=0.75pt]    {$y_{1}$};
\draw (361,34.4) node [anchor=north west][inner sep=0.75pt]    {$x_{1}$};
\draw (181,106.4) node [anchor=north west][inner sep=0.75pt]    {$a_{1}$};
\draw (181,148.4) node [anchor=north west][inner sep=0.75pt]    {$a_{2}$};
\draw (513,107.4) node [anchor=north west][inner sep=0.75pt]    {$b_{1}$};
\draw (513,149.4) node [anchor=north west][inner sep=0.75pt]    {$b_{2}$};

\draw (251.75,177.38) node [anchor=north west][inner sep=0.75pt]  [font=\scriptsize,rotate=-21.32]  {$1.5\ +\ \epsilon $};
\draw (267.45,159.13) node [anchor=north west][inner sep=0.75pt]  [font=\scriptsize,rotate=-32.86]  {$1.5\ +\ \epsilon $};
\draw (268.38,81.8) node [anchor=north west][inner sep=0.75pt]  [font=\scriptsize,rotate=-337.21]  {$1.5$};
\draw (271.12,120.9) node [anchor=north west][inner sep=0.75pt]  [font=\scriptsize,rotate=-326.31]  {$1.5+12N\epsilon $};
\draw (365.97,132.6) node [anchor=north west][inner sep=0.75pt]  [font=\scriptsize]  {$1.5+12N\epsilon $};
\draw (396.83,183.34) node [anchor=north west][inner sep=0.75pt]  [font=\scriptsize,rotate=-323.07]  {$1.5+12N\epsilon $};
\draw (429.73,195.54) node [anchor=north west][inner sep=0.75pt]  [font=\scriptsize,rotate=-333.79]  {$1.5$};
\draw (429.21,69.46) node [anchor=north west][inner sep=0.75pt]  [font=\scriptsize,rotate=-27.89]  {$1.5\ +\ \epsilon $};
\draw (423.21,83.58) node [anchor=north west][inner sep=0.75pt]  [font=\scriptsize,rotate=-39.26]  {$1.5\ +\ \epsilon $};

\end{tikzpicture}
    \caption{The base case of our inductive construction.}
    \label{fig:base_case_lbs}
\end{figure}

We set the distances in round 1 such that there are intervals $I_0=(\varsigma_0,\varsigma'_0)\subset(0,1)$ and $I_1=(\varsigma_1,\varsigma'_1)\subset(0,1)$ such that $\varsigma'_0<\varsigma_1$ and $(x_1,y_1)$ are labeled $(l,1-l)$ in interval $I_l$. In general, there will be $2^{i-1}$ intervals at the end of round $i-1$, any interval $I^{(i-1)}$ will be split into disjoint intervals $I^{(i)}_0,I^{(i)}_1\subset I^{(i-1)}$ where labelings of $\{x_1,\dots,x_{i-1}\}$ match that of $I^{(i-1)}$ and $(x_i,y_i)$ are labeled $(l,1-l)$ in  $I^{(i)}_l$.

Now we set up the edges to achieve these properties. In round 1, we set the distances as follows.
\begin{align*}
    \delta(x_1,a_1)= \delta(y_1,b_2)&=1.5,\\
    \delta(x_1,a_2)= \delta(y_1,b_1)= \delta(x_1,y_1)&=1.5+12N\epsilon,\\
    \delta(x_1,b_1)=\delta(x_1,b_2)=\delta(y_1,a_1)= \delta(y_1,a_2)&=1.5+\epsilon.
\end{align*}
where $\epsilon$ is a small positive quantity such that the largest distance $1.5+12N\epsilon<1.6$. It is straightforward to verify that for $I_0=(0,\frac{1}{2}^{1/\epsilon})$ we have that $(x_1,y_1)$ are labeled $(0,1)$ by determining the values of $\varsigma$ for which the corresponding cut is the min-cut (Figure \ref{fig:base_case_lbs}). Indeed, we seek $\varsigma$ such that $w_{C01}=w(x_1,b_1)+w(x_1,b_2)+w(x_1,y_1)+w(y_1,a_1)+w(y_1,a_2)$ satisfies
\[w_{C01}\le w_{C00}= w(x_1,b_1)+w(x_1,b_2)+w(y_1,b_1)+w(y_1,b_2),\]\[w_{C01}\le w_{C11}=w(x_1,a_1)+w(x_1,a_2)+w(y_1,a_1)+w(y_1,a_2),\]\[w_{C01}\le w_{C10}= w(x_1,a_1)+w(x_1,a_2)+w(x_1,y_1)+w(y_1,b_1)+w(y_1,b_2),\]
which simultaneously hold for $\varsigma<\frac{1}{2}^{1/\epsilon}$. 

Moreover, we can similarly conclude that $(x_1,y_1)$ are labeled $(1,0)$ for the interval $I_1=(\varsigma_1,\varsigma'_1)$ where $\varsigma_1<\varsigma'_1$ are given by the two positive roots of the equation
\[1-2\varsigma^{\epsilon}+2\varsigma^{12N\epsilon}=0.\]
We now consider the inductive step, to set the distances and obtain an inductive proof of the claim above. In round $i$, the distances are as specified.
\begin{align*}
    \delta(x_i,a_1)= \delta(y_i,b_2)&=1.5,\\
    \delta(x_i,a_2)= \delta(y_i,b_1)= \delta(x_i,y_i)&=1.5+12N\epsilon,\\
    \delta(x_i,b_1)=\delta(x_i,b_2)=\delta(y_i,a_1)= \delta(y_i,a_2)&=1.5+\epsilon,\\
    \delta(x_i,y_j)= \delta(y_i,x_j)&=1.5+6(2j-1)\epsilon\;\;\;(1\le j\le i-1),\\
    \delta(x_i,x_j)= \delta(y_i,y_j)&=1.5+12j\epsilon\;\;\;\;\;\;\;\;\;\;\;\;(1\le j\le i-1).
\end{align*}
We denote the (inductively hypothesized) $2^{i-1}$ $\varsigma$-intervals at the end of round $i-1$ by $I_{\mathbf{b}}^{(i-1)}$, where $\mathbf{b}=\{b^{(1)},\dots,b^{(i-1)}\}\in\{0,1\}^{i-1}$ indicates the labels of $x_j,j\in[i-1]$ in $I_{\mathbf{b}}^{(i-1)}$. Min-cuts from round $i-1$ extend to min-cuts of round $i$ depending on how the edges incident on $(x_i,y_i)$ are set (Figure \ref{fig:lbs}). It suffices to consider only those min-cuts where $x_j$ and $y_j$ have opposite labels for each $j$. Consider an arbitrary such min-cut $C_{\mathbf{b}}=(A_{\mathbf{b}},B_{\mathbf{b}})$ of $G_{i-1}$ which corresponds to the interval $I_{\mathbf{b}}^{(i-1)}$, that is $A_{\mathbf{b}}=\{x_j\mid b^{(j)}=0\}\cup\{y_j\mid b^{(j)}=1\}$ and $B_{\mathbf{b}}$ contains the remaining unlabeled nodes of $G_{i-1}$. It extends to $C_{[\mathbf{b}\;0]}$ and $C_{[\mathbf{b}\;1]}$ for $\varsigma\in I_{\mathbf{b}}^{(i-1)}$ satisfying, respectively,
\begin{align*}
    E_{\mathbf{b},0}(\varsigma):= &\;\; 1-2\varsigma^{\epsilon}+F(C_{\mathbf{b}};\varsigma)>0,\\
    E_{\mathbf{b},1}(\varsigma):= &\;\; 1-2\varsigma^{\epsilon}+2\varsigma^{12N\epsilon}+F(C_{\mathbf{b}};\varsigma)<0,
\end{align*}
where $F(C_{\mathbf{b}};\varsigma)=\sum_{z\in A_{\mathbf{b}}}\varsigma^{\delta(x_i,z)}-\sum_{z\in B_{\mathbf{b}}}\varsigma^{\delta(x_i,z)}=\sum_{z\in B_{\mathbf{b}}}\varsigma^{\delta(y_i,z)}-\sum_{z\in A_{\mathbf{b}}}\varsigma^{\delta(y_i,z)}$. If we show that the solutions of the above inequations have disjoint non-empty intersections with $\varsigma\in I_{\mathbf{b}}^{(i-1)}$, our induction step is complete. We will use an indirect approach for this.

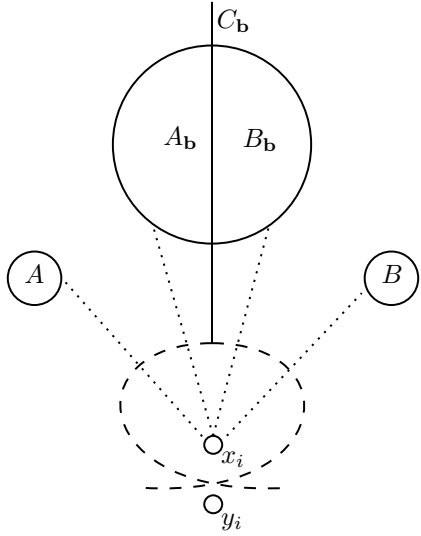
\begin{figure}[t]
    \centering
    \tikzset{every picture/.style={line width=0.75pt}} 

\begin{tikzpicture}[x=0.75pt,y=0.75pt,yscale=-1,xscale=1]

\draw  [dash pattern={on 4.5pt off 4.5pt}]  (299,180.25) .. controls (369,180.25) and (362,259.25) .. (264,253.25) ;
\draw  [dash pattern={on 4.5pt off 4.5pt}]  (299,180.25) .. controls (229,180.25) and (236,259.25) .. (337,253.25) ;

\draw   (249,80) .. controls (249,52.39) and (271.39,30) .. (299,30) .. controls (326.61,30) and (349,52.39) .. (349,80) .. controls (349,107.61) and (326.61,130) .. (299,130) .. controls (271.39,130) and (249,107.61) .. (249,80) -- cycle ;
\draw    (299,8) -- (299,180.25) ;
\draw   (295.09,231.59) .. controls (295.08,229.09) and (297.1,227.06) .. (299.59,227.05) .. controls (302.09,227.04) and (304.13,229.05) .. (304.14,231.55) .. controls (304.15,234.05) and (302.13,236.08) .. (299.64,236.09) .. controls (297.14,236.1) and (295.1,234.09) .. (295.09,231.59) -- cycle ;
\draw   (295.09,261.59) .. controls (295.08,259.09) and (297.1,257.06) .. (299.59,257.05) .. controls (302.09,257.04) and (304.13,259.05) .. (304.14,261.55) .. controls (304.15,264.05) and (302.13,266.08) .. (299.64,266.09) .. controls (297.14,266.1) and (295.1,264.09) .. (295.09,261.59) -- cycle ;
\draw  [dash pattern={on 0.84pt off 2.51pt}]  (225.14,149.55) -- (296.09,229.59) ;
\draw  [dash pattern={on 0.84pt off 2.51pt}]  (306.5,227) -- (374.5,155) ;
\draw  [dash pattern={on 0.84pt off 2.51pt}]  (299.59,227.05) -- (327.5,122) ;
\draw  [dash pattern={on 0.84pt off 2.51pt}]  (299.59,227.05) -- (269.5,122) ;
\draw   (376,147.5) .. controls (376,140.04) and (382.04,134) .. (389.5,134) .. controls (396.96,134) and (403,140.04) .. (403,147.5) .. controls (403,154.96) and (396.96,161) .. (389.5,161) .. controls (382.04,161) and (376,154.96) .. (376,147.5) -- cycle ;
\draw   (196,147.5) .. controls (196,140.04) and (202.04,134) .. (209.5,134) .. controls (216.96,134) and (223,140.04) .. (223,147.5) .. controls (223,154.96) and (216.96,161) .. (209.5,161) .. controls (202.04,161) and (196,154.96) .. (196,147.5) -- cycle ;

\draw (273,69.4) node [anchor=north west][inner sep=0.75pt]    {$A_{\mathbf{b}}$};
\draw (313,70.4) node [anchor=north west][inner sep=0.75pt]    {$B_{\mathbf{b}}$};
\draw (300,10.4) node [anchor=north west][inner sep=0.75pt]    {$C_{\mathbf{b}}$};
\draw (203,139.4) node [anchor=north west][inner sep=0.75pt]    {$A$};
\draw (383,139.4) node [anchor=north west][inner sep=0.75pt]    {$B$};
\draw (302.14,233.5) node [anchor=north west][inner sep=0.75pt]    {$x_{i}$};
\draw (302.14,264.95) node [anchor=north west][inner sep=0.75pt]    {$y_{i}$};

\end{tikzpicture}
    \caption{The inductive step in our lower bound construction for pseudodimension of $\mathcal{H}_{\sigma}$. The min-cut $C_{\mathbf{b}}$ is extended to two new min-cuts (depicted by dashed lines) for which labels of $x_i,y_i$ are flipped, at controlled parameter intervals.}
    \label{fig:lbs}
\end{figure}

For $1\le i\le N$, given $\mathbf{b}=\{b^{(1)},\dots,b^{(i-1)}\}\in\{0,1\}^{i-1}$, let $E_{\mathbf{b},0}$ and $E_{\mathbf{b},1}$ denote the expressions (exponential polynomials in $\varsigma$) in round $i$ which determine labels of $(x_i,y_i)$, in the case where for all $1 \le j < i$, $x_j$ is labeled $b^{(j)}$
(and let $E_{\phi,0},E_{\phi,1}$ denote the expressions for round 1). Let $\varsigma_{\mathbf{b},i} \in (0, 1)$ denote the smallest solution to $E_{\mathbf{b},i} = 0$.
Then we need to show the $\varsigma_{\mathbf{b},i}$’s are well-defined and follow a specific ordering. 
This ordering is completely specified by two conditions:

\begin{enumerate}
    \item[(i)]$\varsigma_{[\mathbf{b}\; 0],1} < \varsigma_{[\mathbf{b}],0} < \varsigma_{[\mathbf{b}],1} < \varsigma_{[\mathbf{b}\; 1],0}$, and 
    \item[(ii)] $\varsigma_{[\mathbf{b} \;0\; \mathbf{c}],1} < \varsigma_{[\mathbf{b} \;1\; \mathbf{d}],0}$
\end{enumerate}

for all $\mathbf{b}, \mathbf{c},\mathbf{d}\in \cup_{i<N}\{0, 1\}^i$ and $|\mathbf{c}| = |\mathbf{d}|$.

First we make a quick observation that all $\varsigma_{\mathbf{b},i}$’s are well-defined and less than $(3/4)^{1/\epsilon}$. To do this, it will suffice to note that $E_{\mathbf{b},i}(0)=1$ and $E_{\mathbf{b},i}(\frac{3}{4}^{1/\epsilon})<0$ for all $\mathbf{b},i$, since the functions are continuous in $(0,\frac{3}{4}^{1/\epsilon})$. This holds because \begin{align*}
    E_{\mathbf{b},0}\left(\frac{3}{4}^{1/\epsilon}\right)<
    E_{\mathbf{b},1}\left(\frac{3}{4}^{1/\epsilon}\right)&=
    1-\frac{3}{2}+\left(\frac{3}{4}\right)^{12N}+F\left(C_{\mathbf{b}};\frac{3}{4}^{1/\epsilon}\right)\\
    &\le -\frac{1}{2}+\left(\frac{3}{4}\right)^{12N}+\sum_{j=1}^{|\mathbf{b}|}\left(\frac{3}{4}\right)^{6j}\left(1-\left(\frac{3}{4}\right)^{6j}\right)\\
    &< -\frac{1}{2}+\sum_{j=1}^{N}\left(\frac{3}{4}\right)^{6j}\\
    &<0
\end{align*}

Let's now consider condition (i). We begin by showing $\varsigma_{[\mathbf{b}],0} < \varsigma_{[\mathbf{b}],1}$ for any $\mathbf{b}$. The exponential polynomials $E_{\mathbf{b},0}$ and $E_{\mathbf{b},1}$ both evaluate to 1 for $\varsigma=0$ (since $|A_{\mathbf{b}}|=|B_{\mathbf{b}}|=|\mathbf{b}|$) and decrease monotonically (verified by elementary calculus) till their respective smallest zeros  $\varsigma_{[\mathbf{b}],0}, \varsigma_{[\mathbf{b}],1}$. But then $E_{\mathbf{b},1}(\varsigma_{[\mathbf{b}],0})=2(\varsigma_{[\mathbf{b}],0})^{12N\epsilon}>0$, which implies $\varsigma_{[\mathbf{b}],0} < \varsigma_{[\mathbf{b}],1}$. Now, to show $\varsigma_{[\mathbf{b}\; 0],1} < \varsigma_{[\mathbf{b}],0}$, note that $E_{[\mathbf{b}\;0],1}(\varsigma)-E_{[\mathbf{b}],0}(\varsigma)=2\varsigma^{12N\epsilon}+\varsigma^{12i\epsilon}-\varsigma^{(12i-6)\epsilon}=\varsigma^{(12i-6)\epsilon}(2\varsigma^{(12(N-i)+6)\epsilon}+\varsigma^{6\epsilon}-1)$ where $1\le i=|\mathbf{b}|+1<N$. Since $\varsigma_{[\mathbf{b}],0}<\frac{3}{4}^{1/\epsilon}$, it follows that $E_{[\mathbf{b}\;0],1}(\varsigma_{[\mathbf{b}],0})<0$, which implies  $\varsigma_{[\mathbf{b}\; 0],1} < \varsigma_{[\mathbf{b}],0}$. Similarly, it is readily verified that $\varsigma_{[\mathbf{b}],1} < \varsigma_{[\mathbf{b}\; 1],0}$, establishing (i).

Finally, to show (ii), note that $E_{[\mathbf{b}\;0\;\mathbf{c}],1}(\varsigma)-E_{[\mathbf{b}\;0\;\mathbf{d}],0}(\varsigma)=2\varsigma^{12N\epsilon}+\varsigma^{12i\epsilon}-\varsigma^{(12i-6)\epsilon}+\varsigma^{12i\epsilon}(F(C_{\mathbf{c}};\varsigma)-F(C_{\mathbf{d}};\varsigma))=\varsigma^{(12i-6)\epsilon}(2\varsigma^{(12(N-i)+6)\epsilon}+\varsigma^{6\epsilon}-1+\varsigma^{6\epsilon}(F(C_{\mathbf{c}};\varsigma)-F(C_{\mathbf{d}};\varsigma)))$. Again, similar to above, we use $\varsigma_{[\mathbf{b}\;0\;\mathbf{d}],0}<\frac{3}{4}^{1/\epsilon}$ in this expression to get $E_{[\mathbf{b}\;0\;\mathbf{c}],1}(\varsigma_{[\mathbf{b}\;0\;\mathbf{d}],0})<0$. Since the exponential polynomials decay monotonically with $\varsigma$ till their first roots, (ii) follows.

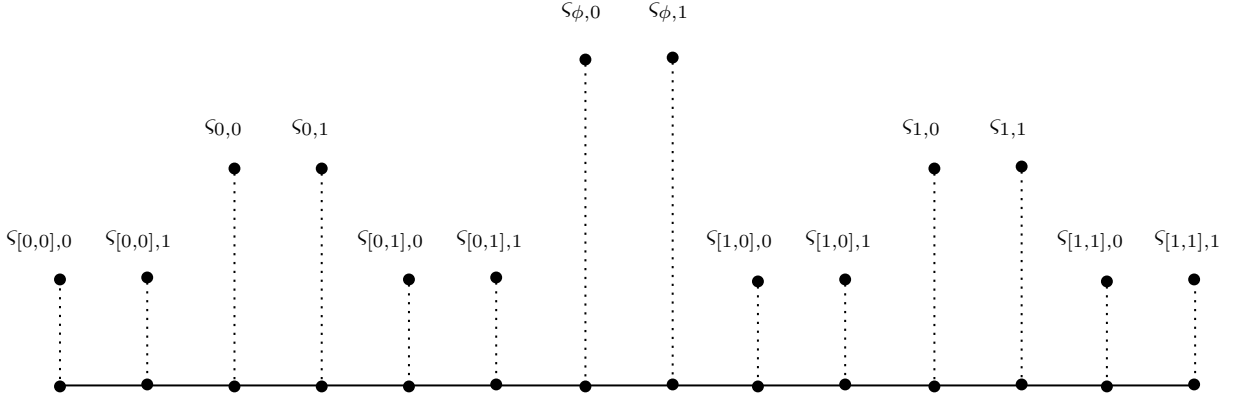
\begin{figure}[t]
    \centering
\tikzset{every picture/.style={line width=0.75pt}} 

\begin{tikzpicture}[x=0.75pt,y=0.75pt,yscale=-1,xscale=1]

\draw    (33,242) -- (606,242) ;
\draw  [fill={rgb, 255:red, 0; green, 0; blue, 0 }  ,fill opacity=1 ] (75,187.5) .. controls (75,186.12) and (76.12,185) .. (77.5,185) .. controls (78.88,185) and (80,186.12) .. (80,187.5) .. controls (80,188.88) and (78.88,190) .. (77.5,190) .. controls (76.12,190) and (75,188.88) .. (75,187.5) -- cycle ;
\draw  [fill={rgb, 255:red, 0; green, 0; blue, 0 }  ,fill opacity=1 ] (31,188.5) .. controls (31,187.12) and (32.12,186) .. (33.5,186) .. controls (34.88,186) and (36,187.12) .. (36,188.5) .. controls (36,189.88) and (34.88,191) .. (33.5,191) .. controls (32.12,191) and (31,189.88) .. (31,188.5) -- cycle ;
\draw  [fill={rgb, 255:red, 0; green, 0; blue, 0 }  ,fill opacity=1 ] (163,132.5) .. controls (163,131.12) and (164.12,130) .. (165.5,130) .. controls (166.88,130) and (168,131.12) .. (168,132.5) .. controls (168,133.88) and (166.88,135) .. (165.5,135) .. controls (164.12,135) and (163,133.88) .. (163,132.5) -- cycle ;
\draw  [fill={rgb, 255:red, 0; green, 0; blue, 0 }  ,fill opacity=1 ] (119,132.5) .. controls (119,131.12) and (120.12,130) .. (121.5,130) .. controls (122.88,130) and (124,131.12) .. (124,132.5) .. controls (124,133.88) and (122.88,135) .. (121.5,135) .. controls (120.12,135) and (119,133.88) .. (119,132.5) -- cycle ;
\draw  [fill={rgb, 255:red, 0; green, 0; blue, 0 }  ,fill opacity=1 ] (251,187.5) .. controls (251,186.12) and (252.12,185) .. (253.5,185) .. controls (254.88,185) and (256,186.12) .. (256,187.5) .. controls (256,188.88) and (254.88,190) .. (253.5,190) .. controls (252.12,190) and (251,188.88) .. (251,187.5) -- cycle ;
\draw  [fill={rgb, 255:red, 0; green, 0; blue, 0 }  ,fill opacity=1 ] (207,188.5) .. controls (207,187.12) and (208.12,186) .. (209.5,186) .. controls (210.88,186) and (212,187.12) .. (212,188.5) .. controls (212,189.88) and (210.88,191) .. (209.5,191) .. controls (208.12,191) and (207,189.88) .. (207,188.5) -- cycle ;
\draw  [fill={rgb, 255:red, 0; green, 0; blue, 0 }  ,fill opacity=1 ] (340,76.5) .. controls (340,75.12) and (341.12,74) .. (342.5,74) .. controls (343.88,74) and (345,75.12) .. (345,76.5) .. controls (345,77.88) and (343.88,79) .. (342.5,79) .. controls (341.12,79) and (340,77.88) .. (340,76.5) -- cycle ;
\draw  [fill={rgb, 255:red, 0; green, 0; blue, 0 }  ,fill opacity=1 ] (296,77.5) .. controls (296,76.12) and (297.12,75) .. (298.5,75) .. controls (299.88,75) and (301,76.12) .. (301,77.5) .. controls (301,78.88) and (299.88,80) .. (298.5,80) .. controls (297.12,80) and (296,78.88) .. (296,77.5) -- cycle ;
\draw  [fill={rgb, 255:red, 0; green, 0; blue, 0 }  ,fill opacity=1 ] (427,188.5) .. controls (427,187.12) and (428.12,186) .. (429.5,186) .. controls (430.88,186) and (432,187.12) .. (432,188.5) .. controls (432,189.88) and (430.88,191) .. (429.5,191) .. controls (428.12,191) and (427,189.88) .. (427,188.5) -- cycle ;
\draw  [fill={rgb, 255:red, 0; green, 0; blue, 0 }  ,fill opacity=1 ] (383,189.5) .. controls (383,188.12) and (384.12,187) .. (385.5,187) .. controls (386.88,187) and (388,188.12) .. (388,189.5) .. controls (388,190.88) and (386.88,192) .. (385.5,192) .. controls (384.12,192) and (383,190.88) .. (383,189.5) -- cycle ;
\draw  [fill={rgb, 255:red, 0; green, 0; blue, 0 }  ,fill opacity=1 ] (516,131.5) .. controls (516,130.12) and (517.12,129) .. (518.5,129) .. controls (519.88,129) and (521,130.12) .. (521,131.5) .. controls (521,132.88) and (519.88,134) .. (518.5,134) .. controls (517.12,134) and (516,132.88) .. (516,131.5) -- cycle ;
\draw  [fill={rgb, 255:red, 0; green, 0; blue, 0 }  ,fill opacity=1 ] (472,132.5) .. controls (472,131.12) and (473.12,130) .. (474.5,130) .. controls (475.88,130) and (477,131.12) .. (477,132.5) .. controls (477,133.88) and (475.88,135) .. (474.5,135) .. controls (473.12,135) and (472,133.88) .. (472,132.5) -- cycle ;
\draw  [fill={rgb, 255:red, 0; green, 0; blue, 0 }  ,fill opacity=1 ] (603,188.5) .. controls (603,187.12) and (604.12,186) .. (605.5,186) .. controls (606.88,186) and (608,187.12) .. (608,188.5) .. controls (608,189.88) and (606.88,191) .. (605.5,191) .. controls (604.12,191) and (603,189.88) .. (603,188.5) -- cycle ;
\draw  [fill={rgb, 255:red, 0; green, 0; blue, 0 }  ,fill opacity=1 ] (559,189.5) .. controls (559,188.12) and (560.12,187) .. (561.5,187) .. controls (562.88,187) and (564,188.12) .. (564,189.5) .. controls (564,190.88) and (562.88,192) .. (561.5,192) .. controls (560.12,192) and (559,190.88) .. (559,189.5) -- cycle ;
\draw  [dash pattern={on 0.84pt off 2.51pt}]  (298.5,80) -- (298.5,242) ;
\draw  [dash pattern={on 0.84pt off 2.51pt}]  (342.5,79) -- (342.5,241) ;
\draw  [dash pattern={on 0.84pt off 2.51pt}]  (253.5,190) -- (253.5,242) ;
\draw  [dash pattern={on 0.84pt off 2.51pt}]  (209.5,191) -- (209.5,243) ;
\draw  [dash pattern={on 0.84pt off 2.51pt}]  (33.5,188.5) -- (33.5,240.5) ;
\draw  [dash pattern={on 0.84pt off 2.51pt}]  (77.5,190) -- (77.5,242) ;
\draw  [dash pattern={on 0.84pt off 2.51pt}]  (385.5,192) -- (385.5,244) ;
\draw  [dash pattern={on 0.84pt off 2.51pt}]  (429.5,191) -- (429.5,243) ;
\draw  [dash pattern={on 0.84pt off 2.51pt}]  (561.5,192) -- (561.5,244) ;
\draw  [dash pattern={on 0.84pt off 2.51pt}]  (606,190) -- (606,242) ;
\draw  [dash pattern={on 0.84pt off 2.51pt}]  (121.5,135) -- (121.5,241) ;
\draw  [dash pattern={on 0.84pt off 2.51pt}]  (165.5,135) -- (165.5,241) ;
\draw  [dash pattern={on 0.84pt off 2.51pt}]  (474.5,135) -- (474.5,241) ;
\draw  [dash pattern={on 0.84pt off 2.51pt}]  (518.5,134) -- (518.5,240) ;
\draw  [fill={rgb, 255:red, 0; green, 0; blue, 0 }  ,fill opacity=1 ] (75,241.5) .. controls (75,240.12) and (76.12,239) .. (77.5,239) .. controls (78.88,239) and (80,240.12) .. (80,241.5) .. controls (80,242.88) and (78.88,244) .. (77.5,244) .. controls (76.12,244) and (75,242.88) .. (75,241.5) -- cycle ;
\draw  [fill={rgb, 255:red, 0; green, 0; blue, 0 }  ,fill opacity=1 ] (31,242.5) .. controls (31,241.12) and (32.12,240) .. (33.5,240) .. controls (34.88,240) and (36,241.12) .. (36,242.5) .. controls (36,243.88) and (34.88,245) .. (33.5,245) .. controls (32.12,245) and (31,243.88) .. (31,242.5) -- cycle ;
\draw  [fill={rgb, 255:red, 0; green, 0; blue, 0 }  ,fill opacity=1 ] (163,242.5) .. controls (163,241.12) and (164.12,240) .. (165.5,240) .. controls (166.88,240) and (168,241.12) .. (168,242.5) .. controls (168,243.88) and (166.88,245) .. (165.5,245) .. controls (164.12,245) and (163,243.88) .. (163,242.5) -- cycle ;
\draw  [fill={rgb, 255:red, 0; green, 0; blue, 0 }  ,fill opacity=1 ] (119,242.5) .. controls (119,241.12) and (120.12,240) .. (121.5,240) .. controls (122.88,240) and (124,241.12) .. (124,242.5) .. controls (124,243.88) and (122.88,245) .. (121.5,245) .. controls (120.12,245) and (119,243.88) .. (119,242.5) -- cycle ;
\draw  [fill={rgb, 255:red, 0; green, 0; blue, 0 }  ,fill opacity=1 ] (251,241.5) .. controls (251,240.12) and (252.12,239) .. (253.5,239) .. controls (254.88,239) and (256,240.12) .. (256,241.5) .. controls (256,242.88) and (254.88,244) .. (253.5,244) .. controls (252.12,244) and (251,242.88) .. (251,241.5) -- cycle ;
\draw  [fill={rgb, 255:red, 0; green, 0; blue, 0 }  ,fill opacity=1 ] (207,242.5) .. controls (207,241.12) and (208.12,240) .. (209.5,240) .. controls (210.88,240) and (212,241.12) .. (212,242.5) .. controls (212,243.88) and (210.88,245) .. (209.5,245) .. controls (208.12,245) and (207,243.88) .. (207,242.5) -- cycle ;
\draw  [fill={rgb, 255:red, 0; green, 0; blue, 0 }  ,fill opacity=1 ] (340,241.5) .. controls (340,240.12) and (341.12,239) .. (342.5,239) .. controls (343.88,239) and (345,240.12) .. (345,241.5) .. controls (345,242.88) and (343.88,244) .. (342.5,244) .. controls (341.12,244) and (340,242.88) .. (340,241.5) -- cycle ;
\draw  [fill={rgb, 255:red, 0; green, 0; blue, 0 }  ,fill opacity=1 ] (296,242.5) .. controls (296,241.12) and (297.12,240) .. (298.5,240) .. controls (299.88,240) and (301,241.12) .. (301,242.5) .. controls (301,243.88) and (299.88,245) .. (298.5,245) .. controls (297.12,245) and (296,243.88) .. (296,242.5) -- cycle ;
\draw  [fill={rgb, 255:red, 0; green, 0; blue, 0 }  ,fill opacity=1 ] (427,241.5) .. controls (427,240.12) and (428.12,239) .. (429.5,239) .. controls (430.88,239) and (432,240.12) .. (432,241.5) .. controls (432,242.88) and (430.88,244) .. (429.5,244) .. controls (428.12,244) and (427,242.88) .. (427,241.5) -- cycle ;
\draw  [fill={rgb, 255:red, 0; green, 0; blue, 0 }  ,fill opacity=1 ] (383,242.5) .. controls (383,241.12) and (384.12,240) .. (385.5,240) .. controls (386.88,240) and (388,241.12) .. (388,242.5) .. controls (388,243.88) and (386.88,245) .. (385.5,245) .. controls (384.12,245) and (383,243.88) .. (383,242.5) -- cycle ;
\draw  [fill={rgb, 255:red, 0; green, 0; blue, 0 }  ,fill opacity=1 ] (516,241.5) .. controls (516,240.12) and (517.12,239) .. (518.5,239) .. controls (519.88,239) and (521,240.12) .. (521,241.5) .. controls (521,242.88) and (519.88,244) .. (518.5,244) .. controls (517.12,244) and (516,242.88) .. (516,241.5) -- cycle ;
\draw  [fill={rgb, 255:red, 0; green, 0; blue, 0 }  ,fill opacity=1 ] (472,242.5) .. controls (472,241.12) and (473.12,240) .. (474.5,240) .. controls (475.88,240) and (477,241.12) .. (477,242.5) .. controls (477,243.88) and (475.88,245) .. (474.5,245) .. controls (473.12,245) and (472,243.88) .. (472,242.5) -- cycle ;
\draw  [fill={rgb, 255:red, 0; green, 0; blue, 0 }  ,fill opacity=1 ] (603,241.5) .. controls (603,240.12) and (604.12,239) .. (605.5,239) .. controls (606.88,239) and (608,240.12) .. (608,241.5) .. controls (608,242.88) and (606.88,244) .. (605.5,244) .. controls (604.12,244) and (603,242.88) .. (603,241.5) -- cycle ;
\draw  [fill={rgb, 255:red, 0; green, 0; blue, 0 }  ,fill opacity=1 ] (559,242.5) .. controls (559,241.12) and (560.12,240) .. (561.5,240) .. controls (562.88,240) and (564,241.12) .. (564,242.5) .. controls (564,243.88) and (562.88,245) .. (561.5,245) .. controls (560.12,245) and (559,243.88) .. (559,242.5) -- cycle ;

\draw (329,47.4) node [anchor=north west][inner sep=0.75pt]    {${{\varsigma _{\phi ,1}}}$};
\draw (285,47.4) node [anchor=north west][inner sep=0.75pt]    {${{\varsigma _{\phi ,0}}}$};
\draw (149,106.4) node [anchor=north west][inner sep=0.75pt]    {${{\varsigma _{0,1}}}$};
\draw (105,106.4) node [anchor=north west][inner sep=0.75pt]    {${{\varsigma _{0,0}}}$};
\draw (501,106.4) node [anchor=north west][inner sep=0.75pt]    {${{\varsigma _{1,1}}}$};
\draw (457,106.4) node [anchor=north west][inner sep=0.75pt]    {${{\varsigma _{1,0}}}$};
\draw (55,162.4) node [anchor=north west][inner sep=0.75pt]    {${{\varsigma _{[ 0,0],1}}}$};
\draw (5,162.4) node [anchor=north west][inner sep=0.75pt]    {${{\varsigma _{[ 0,0],0}}}$};
\draw (232,162.4) node [anchor=north west][inner sep=0.75pt]    {${{\varsigma _{[ 0,1],1}}}$};
\draw (182,162.4) node [anchor=north west][inner sep=0.75pt]    {${{\varsigma _{[ 0,1],0}}}$};
\draw (408,162.4) node [anchor=north west][inner sep=0.75pt]    {${{\varsigma _{[ 1,0],1}}}$};
\draw (358,162.4) node [anchor=north west][inner sep=0.75pt]    {${{\varsigma _{[ 1,0],0}}}$};
\draw (584,162.4) node [anchor=north west][inner sep=0.75pt]    {${{\varsigma _{[ 1,1],1}}}$};
\draw (535,162.4) node [anchor=north west][inner sep=0.75pt]    {${{\varsigma _{[ 1,1],0}}}$};

\end{tikzpicture}
    \caption{Relative positions of critical values of the parameter $\varsigma=e^{-1/\sigma^2}$.}
    \label{fig:base-case-lbs}
\end{figure}
{\it Problem instances}: We will now show the graph instances and witnesses to establish the pseudodimension bound. Our graphs will be $G_i$ from the above construction (padded appropriately such that the min-cut intervals do not change, if we insist each instance has exactly $n$ nodes), and the shattering family $\sigma_b$ ($b=(b_1,\dots,b_N)\in\{0,1\}^N$) will be $2^N$ values of $\sigma$ corresponding to the $2^N$ intervals of $\varsigma$ with distinct min-cuts in $G_N$ described above. To obtain the witnesses, we set the labels so that only the last pair of nodes $(x_i,y_i)$ have different labels (i.e. labels are same for all $(x_j,y_j), j<i$) and therefore the loss function oscillates $2^i$ times as $(x_i,y_i)$ are correctly and incorrectly labeled in alternating intervals. The intervals of successive $G_i$ are nested precisely so that $\sigma_b$ shatter the instances for the above labelings/witnesses. Thus, we have shown that the pseudodimension is $\Omega(N)=\Omega((n-4)/2)=\Omega(n)$.\end{proof}

\section{Active Learning}
\label{app:bal}


\thmpolyal*
\begin{proof}
We first determine values of the graph parameter $\Tilde{\alpha}$ for which two candidate sets $S,T$ have the same heuristic utilities ($u^S=u^T$) in Algorithm \ref{alg: budgeted al}. As noted in the proof of Theorem \ref{thm:dispersion}, $f$ and $f^{L_S}$ are rational polynomials in the graph parameter with degree at most $nd$. For fixed $L_S$, we observe that  $p^{L_S}$ is a rational polynomial in $\Tilde{\alpha}$ with degree $O(nd\ell)$. Putting together, the equation $u^S=u^T$ simplifies to a polynomial equation in $\Tilde{\alpha}$ with degree $O(nd\ell2^{\ell})$. Application of Lemma \ref{lem:bounded} implies that the coefficients of the polynomial equation have bounded joint density, and therefore the roots are $\frac{1}{2}$-dispersed. Accounting for all possible subset pairs $S,T$ and their respective labelings $L_S,L_T$, we obtain a maximum of $(2n)^{2\ell}$ such equations. The discontinuities due to the active learning algorithm $A'$ are therefore dispersed.

At the same time, the discontinuities due to the semi-supervised procedure for label prediction correspond to soft label flips for some $f^{L_S}$. This gives $O((2n)^{\ell})$ polynomial equations of degree at most $nd$ whose roots collectively contain the possible discontinuities in the loss function due to algorithm $A$. On any $\epsilon$-interval we can apply Theorem \ref{thm:poly-roots} to bound the number of roots for each of the polynomial equations and apply Theorem \ref{thm:VC-bound} to show that this implies $\frac{1}{2}$-dispersion.
\end{proof}

\end{document}